%% file: neurips_2026.tex
\documentclass{article}

\PassOptionsToPackage{numbers, compress}{natbib}
\usepackage[main,final]{neurips_2026}

\usepackage[utf8]{inputenc}
\usepackage[T1]{fontenc}
\usepackage{hyperref}
\usepackage{url}
\usepackage{booktabs}
\usepackage{amsfonts}
\usepackage{amsmath}
\usepackage{amssymb}
\usepackage{amsthm}
\usepackage{nicefrac}
\usepackage{microtype}
\usepackage[table]{xcolor}
\usepackage{graphicx}
\usepackage{subcaption}
\usepackage{natbib}
\usepackage{todonotes}
\usepackage{wrapfig}
\usepackage{tcolorbox}
\tcbuselibrary{skins,breakable}
\usepackage{enumitem}
\usepackage{bbm}
\usepackage{aliascnt}
\usepackage{float}
\usepackage{placeins}
\usepackage{fontawesome5}

\usepackage[nameinlink,noabbrev]{cleveref}

\usepackage{caption}
\theoremstyle{plain}
\newtheorem{theorem}{Theorem}[section]

\newaliascnt{proposition}{theorem}
\newtheorem{proposition}[proposition]{Proposition}
\aliascntresetthe{proposition}

\newaliascnt{lemma}{theorem}
\newtheorem{lemma}[lemma]{Lemma}
\aliascntresetthe{lemma}

\newaliascnt{corollary}{theorem}
\newtheorem{corollary}[corollary]{Corollary}
\aliascntresetthe{corollary}

\theoremstyle{definition}
\newaliascnt{definition}{theorem}

\aliascntresetthe{definition}

\newaliascnt{assumption}{theorem}

\aliascntresetthe{assumption}

\theoremstyle{remark}
\newaliascnt{remark}{theorem}

\aliascntresetthe{remark}

\crefname{theorem}{Theorem}{Theorems}
\Crefname{theorem}{Theorem}{Theorems}
\crefname{proposition}{Proposition}{Propositions}
\Crefname{proposition}{Proposition}{Propositions}
\crefname{lemma}{Lemma}{Lemmas}
\Crefname{lemma}{Lemma}{Lemmas}
\crefname{corollary}{Corollary}{Corollaries}
\Crefname{corollary}{Corollary}{Corollaries}
\crefname{definition}{Definition}{Definitions}
\Crefname{definition}{Definition}{Definitions}
\crefname{assumption}{Assumption}{Assumptions}
\Crefname{assumption}{Assumption}{Assumptions}
\crefname{remark}{Remark}{Remarks}
\Crefname{remark}{Remark}{Remarks}

\newcommand{\TV}{\mathrm{TV}}
\newcommand{\KL}{D_{\mathrm{KL}}}

\newcommand{\cC}{\mathcal{C}}

\newcommand{\pstar}{p^\star}
\newcommand{\kNN}{\mathrm{kNN}}

\usepackage{xspace}
\newcommand{\name}
{{\textsc{COMPASS}}\xspace}

\newtcolorbox[auto counter, number within=section]{compassbox}[1][]{
  enhanced,
  breakable,
  title={},
  #1
}

\title{%
    To Align or Not To Align: \\
    Check Your COMPASS Before You Train
}

\author{%
    Fabian Gr\"oger$^{1,2,3,}$\thanks{Joint first authorship}\; , Petar Damjanovi\'c$^{1,}$\footnotemark[1]\; , Maria Brbi\'c$^{1}$
    \\[0.2em]
    $^1$EPFL \; $^2$University of Basel \; $^3$HSLU
}

\begin{document}
 
\maketitle
\frenchspacing

\begin{abstract}
\input{Sections/Abstract}
\end{abstract}

\section{Introduction}
\label{sec:intro}
\input{Sections/Introduction}

\section{Related work}
\label{sec:related}
\input{Sections/RelatedWork}

\section{Preliminaries}
\label{sec:prelim}
\input{Sections/Preliminaries}

\section{Theoretical motivation}
\label{sec:theory}
\input{Sections/Theory_v2}

\section{Cheap-to-compute scores for downstream problems}
\label{sec:practical}
\input{Sections/Practical}

\section{Discussion}
\label{sec:discussion}
\input{Sections/Discussion}

\section*{Acknowledgements}
We thank Shuo Wen, Artyom Gadetsky, Maxim Kodryan, and Jack Naimer for their valuable suggestions, which helped improve the clarity of the manuscript.
We gratefully acknowledge the support of the Swiss National Science Foundation (SNSF) starting grant TMSGI2\_226252/1, SNSF grant IC00I0\_231922, SNSF grant 10.004.411, Swiss Cancer League (KFS-6340-02-2025-R), Kamprad Foundation (Cancer Immunotherapy Networking Program), and the Swiss AI Initiative Large Call (ab006). Maria Brbi\'c is a CIFAR Fellow in the Multiscale Human Program.

\medskip
\small
\bibliography{bibliography}

\newpage

\appendix
\crefalias{section}{appendix}
\crefalias{subsection}{appendix}
\crefalias{subsubsection}{appendix}

\crefname{appendix}{Appendix}{Appendices}
\Crefname{appendix}{Appendix}{Appendices}
\input{Sections/Appendix_v2}

\newpage
\input{checklist.tex}

\end{document}

%% file: Sections/Abstract.tex
\looseness=-1
Combining pretrained models is a common strategy for building AI systems, whether by aligning a vision encoder with a language model or merging fine-tuned checkpoints. However, finding compatible pairs typically requires costly empirical search, as compatibility is known only after alignment or merging. 
We propose \name (COMPatibility ASsessment Score), a theoretically motivated diagnostic that estimates whether a model pair will be compatible \textit{prior} to alignment or merging.
Building on the observation that representational convergence is most robust at the level of local neighborhoods, \name decomposes model compatibility into two cheap-to-compute axes: a pair-specific \emph{ease}, defined as the $k$-NN overlap between the two encoders, and a paradigm-level \emph{ceiling}, capturing the maximum overlap their training paradigms allow. 
Together, these two axes provide a ranking of candidate model pairs and prescribe how to act on each pair: \textit{(i)} prioritize pairs that already agree and whose paradigms are compatible, \textit{(ii)} invest additional compute on pairs with compatible paradigms but low initial agreement, \textit{(iii)} switch the pretraining paradigm when it cannot support sufficient agreement, or \textit{(iv)} deprioritize pairs that score poorly on both axes.
We apply \name to $72$ cross-modal alignment experiments and $392$ task-arithmetic model merges and show that compatibility estimates by COMPASS strongly correlate with downstream success, achieving average Pearson correlation of $0.91$ on cross-modal alignment tasks and  $0.70$ on model merging.  

%% file: Sections/Introduction.tex
Combining two pretrained models, whether by aligning a vision encoder with a language model~\citep{zhai2022lit,li2023blip2,alayrac2022flamingo,merullo2023linearly,koh2023grounding,groger2025limited} or by merging two fine-tuned checkpoints into a single model~\citep{wortsman2022model,matena2022fisher,yadav2024ties,akiba2024evolutionary}, is a central operation in modern transfer pipelines~\citep{zhai2022lit,li2023blip2,alayrac2022flamingo,wortsman2022model, ilharco2023editing}.
In both cases, the most critical decision is which pair of models to combine. We treat both operations as instances of \textit{model-pair compatibility}: the degree of agreement two pretrained models exhibit when aligned or merged. While prior works have proposed transferability metrics for selecting a single pretrained model for fine-tuning~\citep{achille2019task2vec,nguyen2020leep,you2021logme}, the analogous question for model pairs (\textit{i.e.}, whether two pretrained models will combine well) is typically answered only after expensive alignment or merging. A diagnostic that ranks candidate pairs from a small probe set, scoring each model once and reusing these scores across candidate pairs, could avoid most of these runs, reducing evaluation cost from scaling with the number of pairs to scaling with the number of models. 

A growing body of empirical work suggests that such a diagnostic could be achievable. 
The Platonic Representation Hypothesis \citep{huh2024platonic} suggested that sufficiently expressive models converge toward a shared global representation, whereas the Aristotelian Representation Hypothesis \citep{groeger2026aristotelian} showed that the agreement is mainly local: models trained with different objectives and data agree on \emph{who is close to whom}, even when their global geometries differ, which makes it a natural candidate for a diagnostic. However, there is still a gap between this principle and a working model compatibility diagnostic.


\begin{figure}[t]
    \centering
    \includegraphics[width=\linewidth]{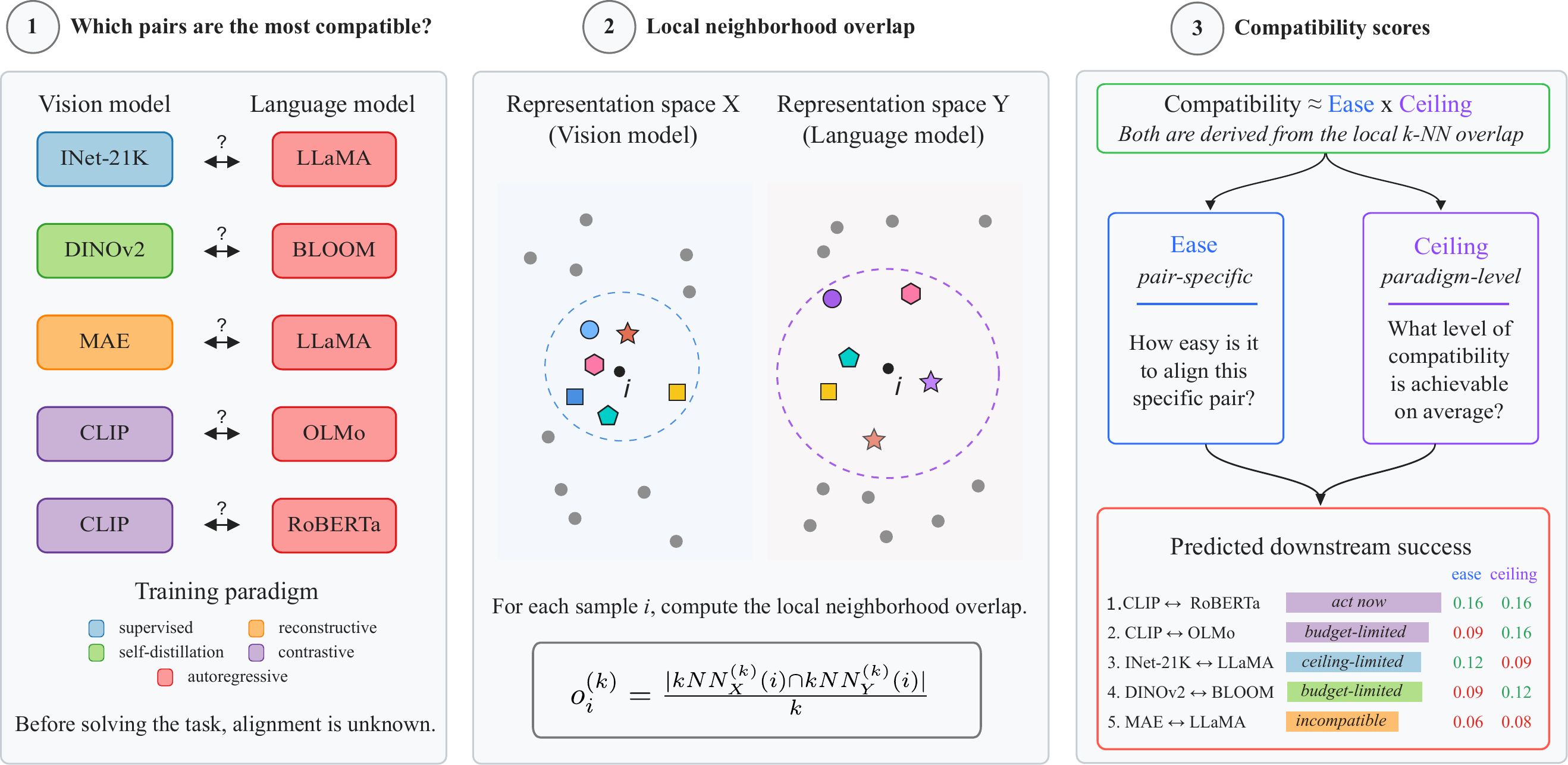}
    \caption{
    Given a library of candidate encoders, the compatibility of a model pair is typically unknown before running the downstream transfer.
    Using \name, we estimate the compatibility based on local neighborhood agreement: for each probe sample, we compare its local $k$-NN overlap across the two representation spaces.
    Aggregating this local overlap across different axes gives two scores computed before alignment or merging.
    \emph{Ease} is a pair-specific score that measures how close the selected pair is at the start of the relevant dynamics: before encoder training in cross-modal alignment, and before weight merging in task arithmetic.
    \emph{Ceiling} is a paradigm-level score that estimates how much agreement the corresponding training paradigms can support on average.
    Candidate pairs are ranked by combining ease and ceiling.
    Each axis is thresholded at its median value: green compatibility scores are above the corresponding median while the red ones are below it.
    The resulting high/low pattern assigns each pair to one of four regimes: \emph{act now} (high ease, high ceiling), \emph{budget-limited} (low ease, high ceiling), \emph{ceiling-limited} (high ease, low ceiling), or \emph{incompatible} (low ease, low ceiling).
    }
    \label{fig:overview}
\end{figure}

\looseness=-1
Here, we close the gap and give a theoretical motivation for why local overlap should predict model-pair compatibility. Motivated by this, we introduce \name (COMPatibility ASsessment Score), a diagnostic that combines two cheap-to-compute axes (\Cref{fig:overview}). The first axis, the \emph{ease} score, asks: How much agreement does a pair exhibit before any training? 
Pairs that already agree on local neighborhoods need only modest training for them to locally agree, whereas pairs with low initial agreement require substantially more training. 
The second axis, the \emph{ceiling} score, asks: How much agreement can the training paradigms of model pairs support in general before training, averaged across models drawn from those paradigms? 
Paradigms with very different training objectives (\textit{e.g.}, reconstructive vs. autoregressive) exhibit limited agreement, regardless of how well aligned the pair is.

We combine ease and ceiling multiplicatively to obtain the continuous \name score, which ranks candidate pairs by favoring those that are strong on both axes. For interpretability, thresholding each axis at its median gives a phase diagram that groups candidate pairs into four regimes: \emph{act now} (high ease, high ceiling), \emph{budget-limited} (low ease, high ceiling), \emph{ceiling-limited} (high ease, low ceiling), and \emph{incompatible} (low ease, low ceiling). 
Act-now pairs are the strongest immediate candidates, and their transfer is expected to be both feasible and efficient. 
For budget-limited pairs, the success may require more data, optimization, or a stronger partner. 
For ceiling-limited pairs, one should switch to a paradigm whose target is closer to the downstream target. 
Incompatible pairs are weak on both axes and should be deprioritized.
The regimes summarize the ranking; the continuous score remains the primary output.

We evaluate \name in two applications: cross-modal alignment and task-arithmetic merging. In cross-modal alignment, we consider $72$ vision-language pairs and show that \name predicts retrieval R@1 with Pearson correlation of $0.90$ on the Flickr30k dataset~\citep{plummer2015flickr30k} and $0.91$ on the COCO dataset~\citep{lin2014microsoft} , compared to the Pearson correlation of $0.63$ and $0.68$ for global CKA~\citep{kornblith2019similarity}. In task-arithmetic merging, we consider $392$ task pairs from $14$ ViT checkpoints fine-tuned on the $8$ image-classification tasks from the standard task-arithmetic benchmark~\citep{ilharco2023editing}. \name predicts mean retention across the $8$ image-classification tasks with Pearson correlation of $0.70$, compared to $0.25$ for global CKA~\citep{kornblith2019similarity}.

\looseness=-1
\textbf{Contributions.}
\textit{(i)} We introduce \name, a diagnostic for selecting model pairs before expensive alignment or merging, combining a pair-specific \emph{ease} score with a paradigm-level \emph{ceiling} score.
\textit{(ii)} We provide a supporting analysis on how quickly the two models will become better aligned, given how close they start, and on a bound that says how close two models' local neighborhoods can ever get, given their training paradigms.
\textit{(iii)} We validate
\name across cross-modal alignment and task-arithmetic merging, demonstrating that it strongly correlates with downstream transfer success and reliably ranks candidate pairs before transfer.  

%% file: Sections/RelatedWork.tex
\vspace{-0.5 em}
\textbf{Representational convergence.}
Representational convergence studies whether independently trained models recover similar structure in their learned representations.
Early evidence showed that different networks trained on the same task learn similar features~\citep{li2016convergent} and that representations from different networks can often be stitched together~\citep{bansal2021revisiting}, suggesting shared structure across independently trained models.
At the scale of foundation models, the Platonic Representation Hypothesis~\citep{huh2024platonic} proposed that sufficiently broad and capable models may converge toward a shared representation. The Aristotelian Representation Hypothesis \citep{groeger2026aristotelian} questioned this global view and showed that only local neighborhood alignment can be empirically observed.
\citet{gupta2026canonicalizing} demonstrated strong convergence within CLIP-style architectures.
Together, these works provide growing empirical evidence that representational convergence is real but more often local than global.
We build on this view but shift the focus from observation to application by using local agreement of trained models to select suitable model pairs for alignment and merging.

\textbf{Structure of training objectives.}
Training objectives shape the geometry of learned representations by specifying the conditional relationships that the model is trained to learn.
This makes them central to model compatibility: two models trained with similar objectives may induce similar neighborhoods, whereas models trained with different objectives may preserve different notions of similarity.
The I-Con framework~\citep{alshammari2025icon} unified contrastive, non-contrastive, and clustering objectives by minimizing Kullback--Leibler (KL)  divergence over conditional distributions. 
We use this view to connect objective-level similarity to local representational agreement.
If two training paradigms induce similar neighborhood structure, their representations can support higher local agreement; if they induce different neighborhoods, downstream transfer is limited even before optimization begins.
This motivates our paradigm-level \textit{ceiling} score, which estimates how much local agreement a pair of training paradigms can support.

\textbf{Multi-modal alignment.}
Multi-modal alignment aims to map representations from different modalities into a shared space where semantically corresponding inputs are close~\citep{zhai2022lit,li2023blip2,alayrac2022flamingo}.
A central difficulty is that different modalities may retain distinct geometry even after alignment. 
\citet{liang2022modality} identified such a persistent ``modality gap'' in CLIP representations, showing that image and text embeddings occupy distinct regions of the shared space even after contrastive training.
Most work on alignment focuses on improving the alignment method~\citep{merullo2023linearly,koh2023grounding} or reducing the amount of paired data required for alignment~\citep{li2025csa,groger2025limited}.
Our goal is complementary; we provide a pre-alignment diagnostic for model selection. 
We show that local $k$-NN overlap is cheap to compute on a probe set, suggesting which frozen encoder pairs are worth aligning before expensive training is run.

\textbf{Weight-space transfer and model merging.}
Model merging combines multiple trained models directly in weight space, with the goal of reusing their capabilities without training a new model from scratch. 
Examples include model soups~\citep{wortsman2022model}, which average fine-tuned weights, task arithmetic~\citep{ilharco2023editing}, which edits behavior via weight-space task vectors, and evolutionary merging~\citep{akiba2024evolutionary}, which searches over merging recipes across tasks.
These methods are most effective when checkpoints being combined remain compatible; otherwise, averaging can interpolate into a region that preserves neither model well.
Existing merging approaches typically discover this only after performing the merge and evaluating the result \citep{wortsman2022model,ilharco2023editing,akiba2024evolutionary}.
In contrast, we study whether mergeability can be estimated beforehand.
For two fine-tuned checkpoints, we measure whether they agree on local neighborhoods and use this agreement to predict merge quality before actual model merging.

%% file: Sections/Preliminaries.tex
We measure agreement using local $k$-NN overlap, the fraction of nearest neighbors that two encoders share for each sample. \citet{groeger2026aristotelian} showed that representational convergence persists across architectures and modalities at the level of local neighborhoods, while global agreement is not observed. \Cref{app:hierarchy} confirms that, on our model library, local overlap correlates with downstream transfer more strongly than global agreement (\textit{e.g.}, CKA~\citep{kornblith2019similarity}) does.

\textbf{Setup.}
We assume a probe set of $N$ samples indexed by $V := \{1, \dots, N\}$. 
Two encoders $X$ and $Y$ embed each sample $i \in V$, producing embeddings of possibly different dimensions.
The theory is agnostic to the choice of probe set; the empirical setup is described in \Cref{app:setup}.

\textbf{Neighborhood overlap.} For each sample $i$ and each encoder, we list the $k$ nearest other samples in the probe set. The neighborhood overlap at $i$ is the fraction of this set that is shared between the two encoders, quantifying how strongly they agree on which samples are neighbors of $i$, irrespective of their order.

Let $\kNN_X^{(k)}(i) \subseteq V \setminus \{i\}$ denote the $k$ nearest neighbors of $i$ in encoder $X$ (ties broken by a fixed deterministic rule).
The per-sample overlap and its dataset-level average are respectively defined as:
\begin{equation}
\label{eq:overlap}
o_i(k) := \frac{|\kNN_X^{(k)}(i) \cap \kNN_Y^{(k)}(i)|}{k},
\quad
O(k) := \tfrac{1}{N}\sum\nolimits_{i=1}^N o_i^{(k)},
\end{equation}
where the per-sample overlap measures agreement around a single sample $i$, and 
its dataset-level average is the standard mutual $k$-NN overlap~\citep{huh2024platonic}, which we refer to as the ``overlap''.
To simplify notation, we fix $k = 10$ throughout the paper as in \citep{huh2024platonic, groeger2026aristotelian} and drop the argument $(k)$. \Cref{app:robustness_k_metric} verifies that the results are unchanged at $k \in \{20, 50, 100\}$.

\textbf{Neighborhood distribution.} To apply information-theoretic tools, we encode the $k$-NN set of each encoder as a probability distribution that places mass $1/k$ on each neighbor and zero elsewhere. We call this the \emph{hard neighborhood distribution} of encoder $X$ at sample $i$, a uniform-mass $k$-NN distribution of the form used in~\citet{alshammari2025icon}:
\begin{equation}
\label{eq:hard_distribution}
p_X(j \mid i) :=
\begin{cases}
\nicefrac{1}{k}, & j \in \kNN_X(i), \\
0, & \text{otherwise},
\end{cases}
\end{equation}
with $p_Y(j \mid i)$ defined analogously.

To ensure that the KL divergences in \Cref{thm:local_convergence} are well-defined, we consider a smoothed version of each hard neighborhood distribution with full support over the non-self indices.
Let $\nu(j \mid i) := \tfrac{1}{N-1}\,\mathbbm{1}\{j \neq i\}$  denote the uniform distribution over the non-self indices $V \setminus \{i\}$. For a fixed small constant $\varepsilon \in (0,1)$, we define
\begin{equation}
\label{eq:soft_distribution}
q_X(j \mid i) := (1-\varepsilon) \, p_X(j \mid i) + \varepsilon \, \nu(j \mid i).
\end{equation}

%% file: Sections/Theory_v2.tex
We motivate \name through sufficient conditions under which local overlap is
informative for compatibility. \Cref{thm:joint_proximity} connects start-condition overlap to
finite-budget convergence, motivating \emph{ease}; \Cref{thm:local_convergence} connects target
discrepancy to achievable overlap, motivating \emph{ceiling}. These results provide
supporting motivation rather than universal guarantees; assumptions and failure
modes are discussed in \Cref{app:limitations} and \Cref{app:assumptions}.

\subsection{Pair-specific ease: per-sample start-condition overlap}
\label{sec:overlap}

We first consider per-sample start-condition overlap and ask whether two models that are compatible in principle can reach high agreement within a finite training budget. Specifically, we consider the start-condition neighborhood overlap, $o_i^{(0)}$, defined as the per-sample overlap from \Cref{eq:overlap}, measured at $t=0$ of the relevant training dynamics, that is before alignment in the cross-modal setting and before weight merging in the task-arithmetic setting.

For a fixed anchor sample $i$, let  $q_W^{(t)}(\cdot \mid i)$ denote the neighborhood distribution induced by model $W \in \{X,Y\}$ at training step $t$, and let $q_W^\star(\cdot \mid i)$ denote the target neighborhood distribution that the model is trying to reach. 
We regard two models as having converged towards each other when the total variation distance (TV) between their induced neighborhood distributions, \mbox{$\delta_i^{(t)} := \TV(q_X^{(t)}(\cdot \mid i),\, q_Y^{(t)}(\cdot \mid i))$},
 falls below the tolerance threshold $\eta$.
We now impose two dynamical assumptions on the training trajectory.
 First, we assume that after $T$ training steps, at least one model lies within total variation distance $\xi$ of its target distribution $q_W^\star(\cdot \mid i)$ with probability at least $p_T$. Second, we suppose that after $t$ training steps the expected per-sample TV gap between the two models' neighborhood distributions, $\delta_i^{(t)}$, is at most $\lambda(t)\,\delta_i^{(0)}$, for a nonincreasing decay function $\lambda$ with $\lambda(t)\to 0$ as $t\to\infty$.
Full statements and discussion of failure modes are in \Cref{app:assumptions}.

\begin{theorem}[Per-sample start-condition overlap controls finite-time convergence]
\label{thm:joint_proximity}
Under the assumptions stated above, for every $\eta > 0$ and for each sample $i$, with probability at least
\begin{equation}
\label{eq:joint_proximity_bound}
\left[p_T - \frac{\lambda(T)(1-\varepsilon)\bigl(1-o_i^{(0)}\bigr)}{\eta}\right]_+
\end{equation}
after $T$ training steps, at least one of the two models is within distance $\xi$ of its target neighborhood distribution and the two models remain within gap $\eta$ of each other. Here $\varepsilon$ is the smoothing parameter in the smoothed $k$-NN distribution from \Cref{eq:soft_distribution}, and $[z]_+ := \max\{z,0\}$.
\end{theorem}
\noindent\emph{Proof.} See \Cref{app:proof_overlap}. \emph{Extension.} Under an additional basin condition (\Cref{app:assumptions}), this extends to asymptotic convergence (\Cref{app:corollary}).  \hfill$\square$

\begin{figure}[t]
\centering
\includegraphics[width=\linewidth]{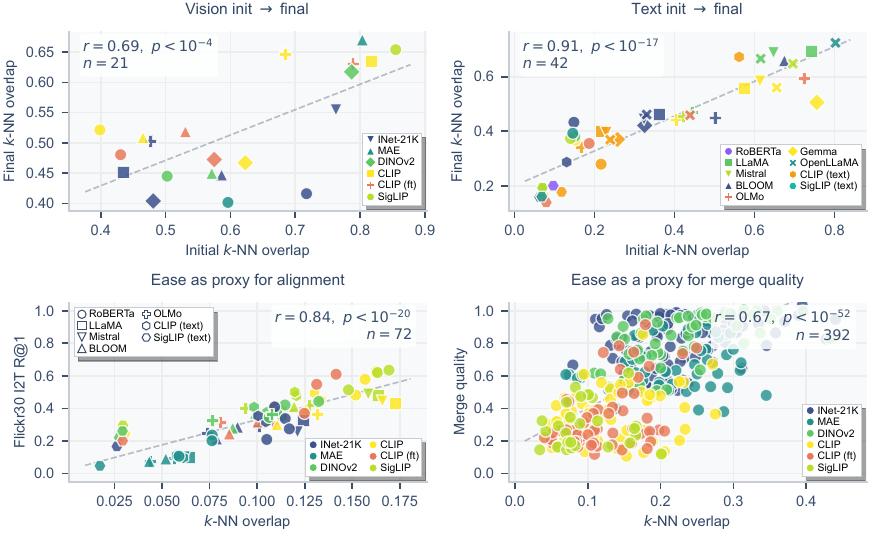}
\caption{
\Cref{thm:joint_proximity} across four settings, each showing a start-condition score against the corresponding endpoint.
\emph{Top:} unimodal validation. Initial random $k$-NN overlap predicts final pretrained overlap across $n = 21$ vision family pairs (left; $r = 0.69$, $p < 10^{-4}$) and $n = 42$ text family pairs (right; $r = 0.91$, $p < 10^{-17}$).
\emph{Bottom left:} cross-modal alignment. Pretrained $k$-NN overlap predicts Flickr30k I2T R@1 after alignment ($r = 0.84$, $p < 10^{-20}$) across $n = 72$ vision-language pairs.
\emph{Bottom right:} model merging via task arithmetic. $k$-NN overlap between fine-tuned checkpoints predicts merge quality ($r = 0.67$, $p < 10^{-52}$) across  $n = 392$ model merges.
Color encodes the vision/text family throughout. 
In the top panels (unimodal) and bottom-left (cross-modal alignment), marker shape encodes the other family of the pair (text family in the cross-modal case).
The bottom-right merging panel uses a single shape, since each point has a single base-model family.
Initialization structure that underlies the unimodal panels is reported in \Cref{app:init_structure}. The same qualitative pattern holds at $k \in \{20,50,100\}$ (\Cref{app:robustness_k_metric}).
}
\label{fig:init_structure}
\end{figure}

The bound motivates using per-sample start-condition overlap as the pair-specific \emph{ease} score: it is measurable before any joint optimization, and higher values reduce the initial inter-model gap $\delta_i^{(0)}=(1-\varepsilon)(1-o_i^{(0)})$, thereby tightening the finite-time guarantee. Pairs with higher per-sample start-condition overlap are therefore easier to make compatible within a fixed training budget.

\looseness=-1
In practice, the start-condition overlap is the per-sample overlap measured at the start of the relevant training dynamics, $o_i^{(0)}$, which differs by setting.
In unimodal convergence, the dynamics goes from a random initialization to the pretrained model, so the per-sample start-condition overlap is the per-sample overlap between two random initializations of the architecture.
In cross-modal alignment, the dynamics goes from the two frozen pretrained models to the trained adapter, so the per-sample start-condition overlap is the pretrained per-sample overlap between those models.
In task-arithmetic merging, the dynamics start from the two fine-tuned checkpoints to their weight-averaged combination, so the per-sample start-condition overlap is the per-sample overlap between the two checkpoints. \Cref{fig:init_structure} shows the per-sample start-condition overlap against the corresponding endpoint across different settings.

\subsection{Paradigm-level ceiling: target discrepancy}
\label{sec:why_local}
We now look into the paradigm-level component of compatibility. Even if two models start from a favorable pair-specific per-sample start-condition overlap, their training objectives \citep{alshammari2025icon} (\emph{e.g.}, contrastive, reconstructive, self-distillation) may result in different neighborhoods. Such disagreement limits the per-sample overlap that the two representations can support, independently of the finite training budget. We formalize this obstruction through a target-discrepancy term and show that it determines the paradigm-level ceiling used in \name. 

For each model $W \in \{X,Y\}$, we define $q_W^{\star}(\cdot \mid i)$ to be the
target conditional distribution over candidate neighbors of anchor sample $i$,
supported on all samples except $i$ itself. The learned smoothed neighborhood
distribution is $q_W^{(t)}(\cdot \mid i)$, and the training objective is represented
as minimizing the KL divergence
$\KL\bigl(q_W^{\star}(\cdot \mid i)\,\|\,q_W^{(t)}(\cdot \mid i)\bigr)$.
Define the \emph{target discrepancy} as:
\begin{equation}
\label{eq:target_discrepancy}
\gamma_i := \TV\bigl(q_X^{\star}(\cdot \mid i),\, q_Y^{\star}(\cdot \mid i)\bigr),
\end{equation}
 where $\gamma_i$ measures how differently the two paradigms define neighbors of sample $i$. Small $\gamma_i$ means that the two objectives induce similar local topology, whereas large $\gamma_i$ means that the objectives favor incompatible neighborhood structures. 

In practice, the targets $q_W^{\star}$ are not directly observable; we only need to estimate $\bar\gamma := \mathbb{E}_i[\gamma_i]$, the \emph{mean target discrepancy} across the probe set. This quantity summarizes how strongly the two training paradigms disagree in the local neighborhoods they induce.

\begin{theorem}[Target discrepancy caps per-sample overlap]
\label{thm:local_convergence}
For each sample $i$ at time step $t$,
\begin{equation}
o_i^{(t)} \le 1 - \frac{1}{1-\varepsilon}\Biggl[\gamma_i -  \sqrt{\tfrac{1}{2}\KL\bigl(q_X^{\star}(\cdot \mid i) \,\|\, q_X^{(t)}(\cdot \mid i)\bigr)} - \sqrt{\tfrac{1}{2}\KL\bigl(q_Y^{\star}(\cdot \mid i) \,\|\, q_Y^{(t)}(\cdot \mid i)\bigr)}\Biggr].
\end{equation}
\end{theorem}

\noindent\emph{Proof.} See \Cref{app:proof_local_convergence}.  \hfill$\square$

Continued optimization drives the two KL terms to zero but cannot reduce $\gamma_i$, so achievable per-sample overlap is bounded by $1 - \gamma_i/(1-\varepsilon)$ even under perfect optimization.
Averaging the bound over samples yields a paradigm-level statement: the mean overlap is bounded above by $1 - \bar\gamma/(1-\varepsilon)$ in the limit of perfect optimization. This motivates an empirical paradigm-level proxy for the achievable overlap, which we call the \emph{ceiling} score. 

\subsection{Implications from combining ease and ceiling}
Together, \Cref{thm:joint_proximity} and \Cref{thm:local_convergence} motivate complementary finite-budget and achievable agreement components of compatibility.

The event in \Cref{thm:joint_proximity} has two parts: at least one model reaches its target within tolerance $\xi$, and the two models remain within gap $\eta$ of each other. If the two models $X$ and $Y$ share the same target $q^\star(\cdot \mid i)$, then the triangle inequality implies that both models are within $\xi + \eta$ of this shared target. More generally, when the targets differ, if model $X$, for example, is within $\xi$ of its target, then model $Y$ is within $\xi + \eta$ of model $X$'s target and within $\xi + \eta + \gamma_i$ of its own target, where \(\gamma_i\) is the target discrepancy defined in
\Cref{eq:target_discrepancy}. Thus, start-condition overlap controls the
finite-time probability of reaching a common neighborhood structure, while any
mismatch between the targets appears as an additional obstruction. 
The bound increases strictly in the per-sample start-condition overlap $o_i^{(0)}$, so pairs that start with more shared neighbors have a tighter finite-time guarantee.

%% file: Sections/Practical.tex
\begin{figure}[b]
\centering
\includegraphics[width=\linewidth]{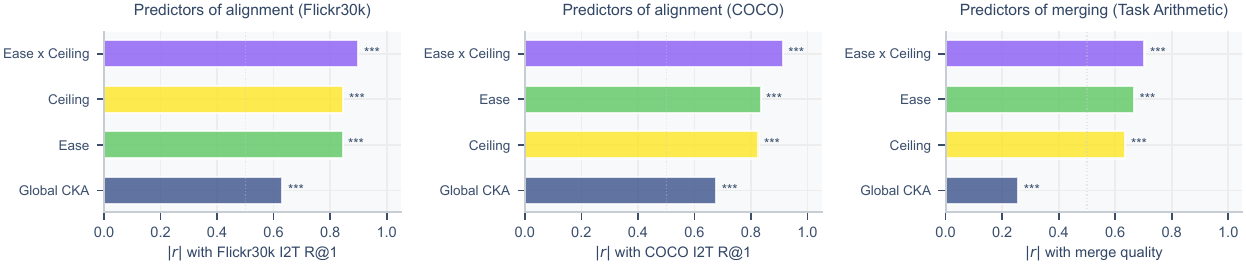}
\caption{
Absolute Pearson correlations between each candidate predictor and the downstream outcome; significance against zero is indicated by $^{*}\!p < 0.05$, $^{**}\!p < 0.01$, $^{***}\!p < 0.001$.
\emph{Left/Center:} cross-modal alignment evaluated on Flickr30k and COCO I2T R@1 ($n = 72$ pairs). Ease (pretrained $k$-NN overlap) and ceiling (paradigm-level cross-modal overlap) dominate, with their product (\emph{Ease}~$\times$~\emph{Ceiling}) reaching $|r| = 0.90$ (Flickr30k) and $|r| = 0.91$ (COCO).
\emph{Right:} task-arithmetic merging against mean retention ($n = 392$).
The product predictor (\emph{Ease}~$\times$~\emph{Ceiling}) reaches $|r| = 0.70$, with ease at $|r| = 0.67$ and ceiling at $|r| = 0.64$ alone.
}
\label{fig:predictors}
\end{figure}

The pair-specific \emph{ease} and paradigm-level \emph{ceiling} scores together provide a measure of model compatibility for the two applications we consider: \emph{(i)} cross-modal alignment, and \emph{(ii)} model merging. We combine them into a simple diagnostic recipe, \hyperref[box:compass]{\name} (COMPatibility ASsessment Score), which ranks candidate model pairs before any transfer is performed. Both quantities are inexpensive to compute, requiring only forward passes over a small probe set.

The probe set should be drawn from the relevant data distribution, \textit{i.e.}, image-text pairs for cross-modal alignment (where the diagnostic measures vision-text agreement) and task-relevant images for merging fine-tuned checkpoints (where it measures agreement between models specialized to those tasks). 
In practice, we use $1024$ image-text pairs from WIT~\citep{srinivasan2021wit} for alignment, and a $2048$-sample probe pooled from one of the eight fine-tuning tasks for merging, here MNIST~\citep{lecunmnist} (\Cref{app:probe_sensitivity} ablates the choice).

Candidate pairs are ranked by the product of ease and ceiling scores. Importantly, feature extraction scales with the number of models rather than candidate pairs. For the $19 \times 10$ alignment grid, \name requires only 29 probe passes instead of 190 downstream alignment runs. In our setup, this translates to more than two orders of magnitude lower wall-clock cost, with full timing details reported in \Cref{app:ceiling_robustness}.

    \begin{compassbox}[label={box:compass}]
  \textbf{\faCompass\enspace COMPASS: Model compatibility diagnostic.} 
  Given a library of candidate models and a way to combine them (\textit{e.g.}, cross-modal alignment or task-arithmetic model merging):
  \begin{enumerate}[wide, labelwidth=!, labelindent=0pt]
      \item \textbf{Compute ease and ceiling on the probe set.} \emph{Ease} is the pair-specific $k$-NN overlap measured at the start of the relevant training dynamics. \emph{Ceiling} is a paradigm-level mean $k$-NN overlap that estimates the upper bound on agreement the two paradigms can support. 
      Both scores lie in $[0, 1]$; higher score is better.
      \item \textbf{Rank pairs by the product of ease and ceiling.} Produce ranking of pairs by combining ease and ceiling scores. Pairs with higher scores should be prioritized.
      \item \textbf{Read the quadrant as a fallback diagnostic.} Splitting each axis at its median yields four regimes: \emph{ (i) act now} (high ease, high ceiling), \emph{(ii) budget-limited} (low ease, high ceiling), \emph{(iii) ceiling-limited} (high ease, low ceiling), and \emph{(iv) incompatible} (low ease, low ceiling). For \emph{budget-limited}, invest in more data, optimization, or a stronger partner. For \emph{ceiling-limited}, switch to a paradigm whose target is closer to the downstream target. Deprioritize \emph{incompatible} pairs.
  \end{enumerate}
  \end{compassbox}

\Cref{fig:predictors} compares the correlation between candidate predictors and the downstream outcome in cross-modal alignment and model merging via task arithmetic. In cross-modal alignment, we measure the downstream outcome as image-to-text Recall@1 (R@1) after training a CSA adapter, evaluated on Flickr30k and COCO datasets.
The product \emph{ease}~$\times$~\emph{ceiling} reaches Pearson correlation of $|r| = 0.90$ with R@1 on the Flickr30k I2T dataset and $|r| = 0.91$ with R@1 on the COCO I2T dataset, compared to $0.84/0.83$ for ease alone and $0.85/0.83$ for ceiling alone, respectively.
In task-arithmetic model merging, we measure the downstream outcome as the mean of per-task accuracy retained relative to the unmerged checkpoints after weight averaging, evaluated across 392 task-pair merges from 14 pretrained ViT backbones spanning five pretraining paradigms.
The product \emph{ease}~$\times$~\emph{ceiling} reaches $|r| = 0.70$ with merge quality, compared to $0.67$ for ease alone and $0.64$ for ceiling alone.
Global CKA, in contrast, correlates only weakly with downstream outcomes in both applications ($0.63/0.68$ for alignment, $0.25$ for merging).

Because the practical goal is model selection, we next ask whether this advantage is preserved at the level of pair rankings. We compare \name with the alternative predictors using Spearman $\rho$ and Kendall $\tau$. \name remains the strongest predictor, reaching Spearman $\rho = 0.91/0.93$ on Flickr30k/COCO and $0.70$ for merging; full results are reported in \Cref{app:ranking_selection}. Under budgeted evaluation, \name finds the optimal alignment pair immediately and the optimal merge within the first 5 of 392 candidates (\Cref{app:ranking_selection}).

We further test whether this advantage depends on the particular set of predictors considered in \Cref{fig:predictors}. Comparisons with metadata predictors, alternative representational similarities, and learned metadata regressors are reported in \Cref{app:baselines}. \name outperforms the baselines computed before alignment or merging; learned regressors can perform strongly when trained on completed alignment or merging outcomes, but degrade under held-out families or paradigms.

\subsection{\name for cross-modal alignment}
\label{sec:practical_align}

\begin{figure}[b]
\centering
\includegraphics[width=\linewidth]{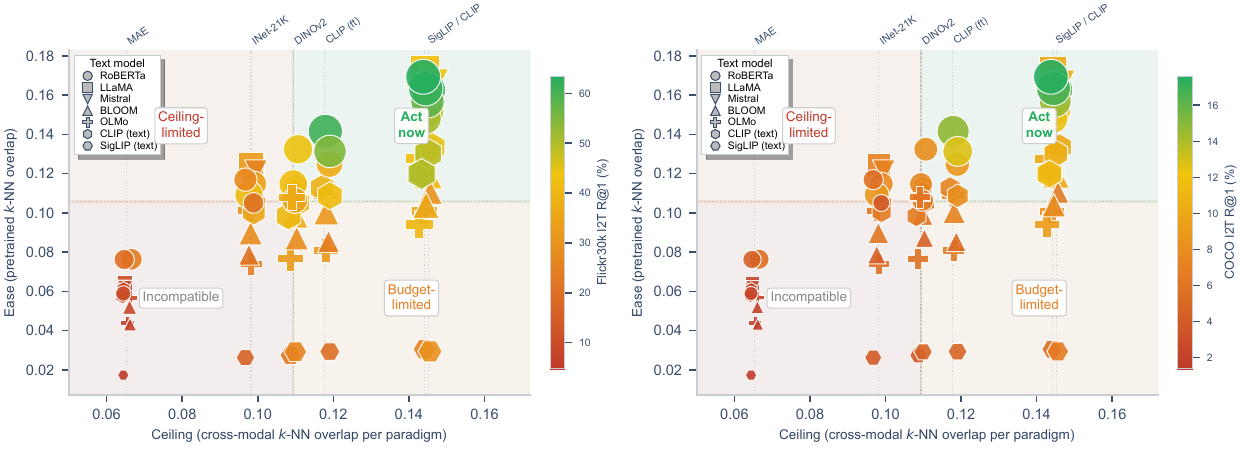}
\caption{
Alignability phase diagram ($n = 72$ pairs). $x$-axis: paradigm-level mean cross-modal $k$-NN overlap (ceiling). $y$-axis: pretrained $k$-NN overlap (ease). Color and size encode retrieval I2T R@1 after alignment. \emph{Left:} Flickr30k. \emph{Right:} COCO. The quadrant structure is consistent across both benchmarks}
\label{fig:phase}
\end{figure}

Cross-modal alignment trains a learned adapter on paired image-text data so that a frozen vision encoder and a frozen language model produce compatible representations~\citep{zhai2022lit,li2023blip2,alayrac2022flamingo,merullo2023linearly,koh2023grounding,groger2025limited}.
Following the setup of \citet{groger2025limited}, we connect the two encoders with a small learned adapter (CSA~\citep{li2025csa}) trained on MS-COCO~\citep{lin2014microsoft} dataset.
Downstream performance is measured on two retrieval benchmarks Flickr30k~\citep{plummer2015flickr30k} and COCO~\citep{lin2014microsoft} datasets using image-to-text R@1 after alignment.
We evaluate $72$ vision-language pairs spanning $19$ vision models ($6$ paradigm families) and $10$ text encoders ($7$ families); see \Cref{app:setup_library} for the full library.

For alignment, \emph{ease} is the pretrained cross-modal $k$-NN overlap of the vision-text pair, while \emph{ceiling} averages the vision encoder's overlap across the reference language-model library; both are measured on the 1024-sized subset from WIT as in \citep{huh2024platonic} before any alignment training (\Cref{app:score_computation}). Intuitively, a vision paradigm whose pretraining target is already linked to language (contrastive vision-language) has a high ceiling, whereas one whose target is pixel reconstruction (masked autoencoding) has a low ceiling.

The ease axis separates pairs whose frozen encoders already agree on neighborhoods from those that do not.
The ceiling axis separates contrastive vision-language paradigms (CLIP $0.140$, SigLIP $0.134$) from CLIP fine-tuned ($0.112$), self-distillation ($0.109$), discriminative ($0.100$), and reconstructive ($0.058$).
\Cref{fig:phase} shows the resulting phase diagram for both retrieval benchmarks; the quadrant structure is consistent across both.
Mean Flickr30k I2T R@1 separates the regimes cleanly: \emph{act now} ($n = 28$) reaches $0.46 \pm 0.09$, \emph{ceiling-limited} ($n = 8$) $0.34 \pm 0.05$, \emph{budget-limited} ($n = 10$) $0.30 \pm 0.06$, and \emph{incompatible} ($n = 26$) $0.21 \pm 0.11$, with the same ordering on COCO.
A Kruskal--Wallis test~\citep{kruskal1952use} rejects the null that the four regimes share the same R@1 distribution (Flickr30k: $H = 45.3$, $p < 10^{-9}$; COCO: $H = 51.5$, $p < 10^{-10}$).
\Cref{app:regime_boxplots} reports the full per-regime distribution.
Contrastive vision-language paradigms (CLIP, SigLIP, CLIP-FT) and the largest DINOv2 variants populate \emph{act now}, reaching up to $63\%$ Flickr R@1; smaller DINOv2 and ImageNet-21K models fall into \emph{ceiling-limited}; MAE (mean Flickr R@1 $0.11$, range $5$-$24\%$) along with the smallest ImageNet-21K and DINOv2 models occupy \emph{incompatible}. The unusually low-overlap cases associated with the SigLIP text tower are analyzed in \Cref{app:siglip_text}.

We next test whether this relationship is driven by coarse differences between model families or by details of the alignment setup. First, ease remains predictive when the vision-model family is held fixed: within each of the six vision families, ease correlates positively with R@1 across the text partners, with a median Spearman $\rho = 0.80$ (\Cref{app:within_family}), showing that its signal is not explained only by broad differences between pretraining paradigms. Second, the diagnostic tracks changes even within a fixed architecture: classification fine-tuning lowers the cross-modal ceiling at all three tested CLIP scales, with retrieval changing in the same direction (\Cref{app:training_stage}). Finally, the resulting pair ranking is not specific to CSA: on 24 representative pairs, \name remains predictive under trained Linear, MLP, and LoRA-style adapters, and outperforms global CKA in each case (\Cref{app:adapter_robustness}).

The training-budget experiments further clarify the meaning of ceiling. Across 18 alignment pairs, its correlation with downstream performance increases from $\rho=0.79$ after 10 epochs to $0.92$ after 300 epochs, as low-ceiling pairs saturate while higher-ceiling pairs continue toward stronger outcomes (\Cref{app:training_budget}). Ceiling therefore captures the level of performance that a transfer setting can ultimately support, rather than the marginal improvement obtained from increasing the training budget over an arbitrary interval.

\subsection{\name for task-arithmetic model merging}
\label{sec:practical_merge}

\looseness=-1
Combining the weights of multiple fine-tuned models is a common alternative to joint multi-task training~\citep{wortsman2022model,matena2022fisher,yadav2024ties,akiba2024evolutionary}.
We focus on the \emph{task arithmetic}~\citep{ilharco2023editing}: two checkpoints of the same base model, each fine-tuned on a different downstream task, are combined into a single model whose weights are the average of the two, with the goal of inheriting the capabilities of both.
The downstream metric is mean retention across the two original tasks after merging.
We fine-tune $14$ models (9 ViT-B/16, 4 ViT-L/14-16, 1 ViT-T/16) from $5$ pretraining paradigms on $8$ tasks from the standard benchmark~\citep{ilharco2023editing} and merge all within-paradigm pairs (392 model merges in total).
Details and the full task-arithmetic multi-task extension are in \Cref{app:merge_extras}.

For task-arithmetic merging, \emph{ease} is the pre-merge $k$-NN overlap between the two fine-tuned checkpoints, while \emph{ceiling}  measures specialization of their pretrained base model relative to the reference vision-model library (\Cref{app:score_computation}).

\begin{wrapfigure}[19]{r}{0.5\linewidth}
\vspace{-1.5em}
\centering
\includegraphics[width=\linewidth]{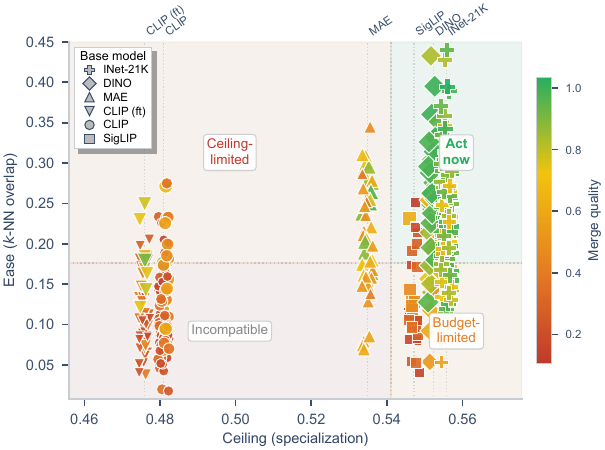}
\caption{
Mergeability phase diagram ($n = 392$ pairs).
$x$-axis: paradigm \emph{specialization}, defined as $1 - $ mean across-paradigm $k$-NN overlap from the convergence experiments (ceiling).
$y$-axis: pair-specific pre-merge $k$-NN overlap (ease).
Color and size encode merge quality.
}
\label{fig:merge_phase}
\vspace{-1.5em}
\end{wrapfigure}

If the two checkpoints already agree on neighborhood structure, weight averaging is unlikely to interpolate destructively, consistent with the linear mode connectivity literature~\citep{garipov2018loss,frankle2020linear}; whereas if they disagree, averaging interpolates into a region that neither checkpoint occupied. Specialized paradigms, such as those whose features differ most from other paradigms, produce smaller, more consistent task vectors under fine-tuning, so weight averaging drifts less.

The four regimes mirror those in \Cref{sec:practical_align}, applied to fine-tuned checkpoints rather than pretrained encoders.
Discriminative, self-distillation, and reconstructive paradigms populate the \emph{act now} regime; the \emph{ceiling-limited} and \emph{incompatible} regimes concentrate on contrastive vision-language paradigms, whose features change substantially under classification fine-tuning.
Mean retention again separates the regimes cleanly (Kruskal--Wallis $H = 198.2$, $p < 10^{-42}$): \emph{act now} ($n = 129$) at $0.84 \pm 0.16$, \emph{budget-limited} ($n = 67$) at $0.57 \pm 0.26$, \emph{ceiling-limited} ($n = 67$) at $0.56 \pm 0.23$, and \emph{incompatible} ($n = 129$) at $0.35 \pm 0.17$ (\Cref{app:regime_boxplots}).
Budget-limited and ceiling-limited pairs reach almost the same mean retention ($0.57$ vs.\ $0.56$), yet they call for different actions: more optimization or a stronger partner for the first, a different paradigm for the second. The regime label carries this distinction, the score alone does not.

%% file: Sections/Discussion.tex
In this work, we show that model-pair compatibility can be estimated before running an expensive transfer procedure. 
A small probe set is enough to measure whether two models agree on local neighborhoods, and this local agreement provides a strong signal of downstream transfer success.

The key idea is to separate compatibility into two parts. \emph{Ease} measures whether a specific pair already starts close: it is high when the two encoders assign similar neighbors to the same examples before alignment or merging.
\emph{Ceiling} measures whether the corresponding training paradigms are compatible with the application: it is high when the paradigms leave enough room for strong agreement to emerge.
The combination of scores yields a simple phase diagram that distinguishes pairs worth using immediately from pairs that are budget-limited, ceiling-limited, or unlikely to succeed.

\looseness=-1
This two-axis view explains why the same model family can behave differently across applications. 
Contrastive vision-language models are favorable for cross-modal alignment, where their pretraining objective is close to the downstream target, but unfavorable for task-arithmetic merging across classification tasks. 
Discriminative and self-distillation models show the reverse pattern. 
Thus, the relevant question is not whether a paradigm is universally good, but whether its objective is compatible with the transfer operation at hand.

These conclusions are robust to both the composition of the model library and the precise estimation of the ceiling. Recomputing \name while leaving out individual families or paradigms preserves correlations of $0.85$--$0.93$ for alignment and $0.60$-$0.72$ for merging on the remaining candidate pools (\Cref{app:heldout}). Likewise, replacing pair-specific ceiling estimates with a single shared value per paradigm, or perturbing them with moderate noise, leaves the ranking largely unchanged, whereas shuffling the paradigm ordering removes this benefit (\Cref{app:ceiling_robustness}). The four regimes in our phase diagram should therefore be viewed as an interpretable summary rather than strict categories; for finer-grained comparisons, the continuous \emph{ease $\times$ ceiling} score is more informative. We discuss further limitations of \name in \Cref{app:limitations}.

\name can help practitioners spend their compute budget where it is most likely to pay off. By ranking candidate pairs using only forward passes on a small probe set, it reduces the need to align or merge every possible combination. Its two-axis diagnostic also indicates whether to invest further in a pair or reconsider the pretraining paradigm. As pretrained model libraries grow, such inexpensive compatibility checks can make model selection more systematic and guide costly experiments toward the most promising combinations.

%% file: Sections/Appendix_v2.tex
\section{Limitations}
\label{app:limitations}

\name might generalize to other domains involving interactions between two learned representations or models, but we leave such extensions to future work.

A second limitation concerns the definition of the ceiling proxy. In our experiments, we instantiate this score separately for the cross-modal alignment and for the task-arithmetic merging. More generally, the ceiling proxy depends on the application and requires domain-specific knowledge; it cannot be expressed as a single universal formula.

Our complete 72-pair cross-modal sweep uses CSA. To test whether the observed compatibility ranking is specific to this alignment procedure, we additionally evaluate trained Linear, MLP, and LoRA-style adapters on 24 representative pairs spanning all vision paradigms and find consistent pair rankings across alignment methods (\Cref{app:adapter_robustness}). 

Our dynamical assumptions from \Cref{sec:overlap} are sufficient conditions, not necessary ones, and may not hold for all training regimes. A2 is strongest when the training dynamics preserve the initial ordering of inter-encoder gaps and drive them down through a finite-time decay function $\lambda(t)$. For models with very different optimizers, data distributions, or objectives, this decay may be slow, yielding a loose bound. \Cref{app:noisy_contraction} relaxes A2 to allow an additive noise floor that captures realistic SGD noise.

Finally, our phase diagram should be interpreted as a finite-budget diagnostic rather than an impossibility result. Pairs with low ease and low ceiling are predicted to be poor candidates under a fixed alignment or merging recipe and realistic compute and data budgets. However, with substantially more alignment data, larger adapters, longer training, or a different optimization objective, even pairs that are initially classified as \emph{incompatible} may eventually achieve strong downstream performance. Thus, the proposed regimes describe the expected efficiency and practicality of transfer, not a fundamental impossibility claim.

\section{Experimental setup}
\label{app:setup}

\subsection{Model libraries and measurement protocol}
\label{app:setup_library}

We use three model libraries across the experiments. The \emph{convergence library} (\Cref{app:setup_library_convergence}) supports the unimodal init-to-final panels of \Cref{fig:init_structure}. The \emph{alignment library} (\Cref{app:setup_library_alignment}) supports the cross-modal experiments of \Cref{sec:practical_align}. The \emph{merging library} (\Cref{app:setup_library_merging}) supports the merging experiments of \Cref{sec:practical_merge}.

Model features are computed on a $1{,}024$-pair probe from WIT~\citep{srinivasan2021wit}, following \citet{huh2024platonic}, for both the convergence and alignment experiment.
For the merge experiment, we extract features from a $2048$-sample probe pooled from one of the eight fine-tuning tasks for merging, here MNIST~\citep{lecunmnist} (\Cref{app:probe_sensitivity} ablates the choice).
We extract CLS-token activations on the vision side and mean-pooled hidden states on the language side from all layers, apply outlier-robust preprocessing (quantile clipping at $q=0.95$), and select the layer combination that maximizes the calibrated overlap per model pair, following the layer-selection protocol of \citet{huh2024platonic}.
SigLIP variants in \texttt{timm} are trained without a class token and use an attention-pool head, so for these we substitute mean pooling over patch tokens, which produces alignment scores consistent with CLIP across all comparisons.
We evaluate mutual $k$-NN overlap as a local metric and linear CKA as a global metric, both calibrated with the permutation procedure of \citet{groeger2026aristotelian}, which controls for chance overlap by comparing the measured similarity against a null distribution from $K = 200$ random permutations of the embedding rows at significance level $\alpha = 0.05$.

\subsubsection{Convergence library}
\label{app:setup_library_convergence}

The convergence library contains 20 vision models and 17 language models.
The 20 vision models~\citep{wightman2019timm} span six paradigms: ImageNet-21K classification~\citep{ridnik2021imagenet21k} with ViT-T/S/B/L~\citep{dosovitskiy2021image}, masked autoencoders MAE-B/L/H~\citep{he2022masked}, DINOv2 self-distillation models at four scales (S/B/L/G,~\citealp{oquab2024dinov2}), CLIP-B/L/H~\citep{radford2021learning}, CLIP-B/L/H fine-tuned on ImageNet-12K, and SigLIP-B/L/SO400M~\citep{zhai2023sigmoid}.
The 17 language models span seven families: BLOOM at 560M to 7.1B parameters~\citep{lescao2023bloom}, OpenLLaMA at 3B to 13B parameters~\citep{geng2023openllama}, the LLaMA family at 7B and 13B~\citep{touvron2023llama} together with Llama-3-8B~\citep{grattafiori2024llama3}, OLMo at 1B and 7B parameters~\citep{groeneveld2024olmo}, Gemma at 2B and 7B parameters~\citep{gemmateam2024gemma}, Mistral at 7B parameters~\citep{jiang2023mistral}, and the encoder model RoBERTa-large~\citep{liu2019roberta}.
The language panel spans heterogeneous architectural choices (ALiBi vs.\ RoPE positional encoding, multi-query vs.\ grouped-query vs.\ multi-head attention, sliding-window attention, RMSNorm vs.\ LayerNorm, GeGLU vs.\ SwiGLU activations), so the within- and across-family comparisons are not driven by a single dominant design.
For the initialization experiments we create three random variants per model with different seeds.

\subsubsection{Alignment library}
\label{app:setup_library_alignment}

The cross-modal alignment experiments use 19 vision models and 10 text encoders, evaluated on 72 vision-language pairs.
The vision side reuses the convergence library minus DINOv2-Giant, which is excluded because aligning a 1.1B-parameter encoder is too expensive for the per-pair adapter sweep.
The 10 text encoders span seven families: OLMo-1B/7B~\citep{groeneveld2024olmo}, BLOOMz-560M/3B~\citep{lescao2023bloom}, LLaMA-7B~\citep{touvron2023llama}, Mistral-7B~\citep{jiang2023mistral}, RoBERTa-large~\citep{liu2019roberta}, and the text towers of CLIP-B/L~\citep{radford2021learning} and SigLIP-B~\citep{zhai2023sigmoid}.

The 72 pairs sample the $19 \times 10 = 190$ grid in a matchmaking design.
RoBERTa-large is paired with all 19 vision models, anchoring every vision scale and paradigm at one common text encoder.
The four base-size representatives of the four widest vision paradigms (ViT-21K-B, MAE-B, CLIP-B, SigLIP-B) are paired with all 10 text encoders, fully resolving the cross-paradigm overlap matrix at base scale.
Two further base-size representatives (DINOv2-B and CLIP-FT-B) are paired with 8 text encoders, omitting the LLaMA-7B and Mistral-7B columns.
DINOv2-L is paired with the four short-context text encoders only (RoBERTa-large and the three vision-language text towers).
Pairs in the intersection are counted once, summing to 72.

\subsubsection{Merging library}
\label{app:setup_library_merging}

The task-arithmetic merging experiments use a separate library of 14 ViT backbones loaded from public \texttt{timm} checkpoints~\citep{wightman2019timm}, chosen to match the standard task-arithmetic benchmark of \citet{ilharco2023editing}.
The 14 backbones span five paradigms: contrastive vision-language (CLIP-B-LAION, CLIP-B-OpenAI, SigLIP-B, CLIP-L-LAION), contrastive vision-language fine-tuned on classification (CLIP-FT-B, CLIP-FT-L), self-distillation (DINO-B~\citep{caron2021emerging} and DINOv3-B~\citep{simeoni2025dinov3}), reconstructive (MAE-B, MAE-L), and discriminative (ViT-T, two ViT-B variants from the \texttt{augreg} and \texttt{orig} recipes, and ViT-L, all on ImageNet-21K~\citep{ridnik2021imagenet21k}).
The self-distillation paradigm is represented by DINO and DINOv3 instead of DINOv2; the substitution is operational, since publicly available DINOv2 \texttt{timm} base checkpoints use patch14 while the rest of the merging base library is patch16.

\subsubsection{Operational computation of ease and ceiling}
\label{app:score_computation}

\paragraph{Cross-modal alignment.}
Let $O(a,b)$ denote the mutual $k$-NN overlap between encoders $a$ and $b$ on the shared probe set.
For a vision encoder $b$ and text encoder $v$,
\[
E_{\mathrm{align}}(b,v)=O(b,v).
\]
Let $\mathcal{L}_{\mathrm{text}}$ be the reference library of text encoders. We define
\[
C_{\mathrm{align}}(b)
=
\frac{1}{|\mathcal{L}_{\mathrm{text}}|}
\sum_{v'\in\mathcal{L}_{\mathrm{text}}}
O(b,v').
\]
Thus, all candidate pairs containing the same vision encoder receive the same ceiling, while ease remains pair-specific.

\paragraph{Task-arithmetic merging.}
For two task-specific checkpoints $b_{t_1}$ and $b_{t_2}$ derived from the same pretrained base model $b$,
\[
E_{\mathrm{merge}}(b_{t_1},b_{t_2})
=
O(b_{t_1},b_{t_2}).
\]
Let $\mathcal{L}_{\mathrm{vision}}$ be the reference library of pretrained vision models. We define
\[
C_{\mathrm{merge}}(b)
=
1-
\frac{1}{|\mathcal{L}_{\mathrm{vision}}|-1}
\sum_{\substack{v\in\mathcal{L}_{\mathrm{vision}}\\v\neq b}}
O(b,v).
\]
Thus, all merges derived from the same pretrained base model receive the same ceiling, while ease depends on the two fine-tuned checkpoints.

\subsection{Cross-modal alignment protocol}
\label{app:setup_alignment}

We align each vision-language pair with CSA~\citep{li2025csa}, a closed-form adapter based on normalized canonical correlation analysis. Both encoders are kept frozen. The adapter consists of two linear projections that map the vision and language features into a shared $512$-dimensional canonical subspace, and is fitted in closed form by canonical correlation on paired image-caption embeddings. Since the solution is closed form, there is no iterative optimization and no learning rate, batch size, or epoch budget to set. We use a canonical dimension of $512$, no equal-weight constraint between the two projections, and no regularization on the canonical solution.

For each pair, we select a single combination of one vision layer and one language layer. Following the layer-selection protocol of \citet{groger2025limited}, we sample a probe of $5{,}000$ paired image-caption embeddings from the MS-COCO 2014 training split~\citep{lin2014microsoft} and choose the layer combination with the highest mutual $k$-NN overlap at $k = 10$. The canonical correlation is then fitted on paired image-caption embeddings from the MS-COCO 2014 training split at the selected layer combination. No augmentation is applied to the embeddings.

We report image-to-text retrieval Recall@$1$ on the MS-COCO 2014 validation split and on the Flickr30k test split~\citep{plummer2015flickr30k}. Recall@$5$, Recall@$10$, and the text-to-image direction are also computed but are not reported in the main paper. Retrieval uses cosine similarity between $\ell_2$-normalized aligned features with an evaluation batch size of $256$. Each pair is run once with a random seed of $42$. The closed-form solution is deterministic given its input features, so no additional seeds are required.

\subsection{Task-arithmetic merging protocol}
\label{app:setup_merging}

Following the standard task-arithmetic benchmark of \citet{ilharco2023editing}, we use the eight image-classification tasks: MNIST, SVHN, GTSRB, DTD, EuroSAT, RESISC45, Stanford Cars, and SUN397. The 14 pretrained backbones are the merging library of \Cref{app:setup_library_merging}. For every base model and every task we fine-tune one checkpoint, producing $14 \times 8 = 112$ task-specialized models in total.

Fine-tuning follows the standard end-to-end recipe in which all backbone parameters are updated together with the classifier head, with no frozen layers. We use AdamW~\citep{loshchilov2019decoupled} with weight decay $0.01$, the cross-entropy loss, and a cosine annealing schedule across the full training run~\citep{loshchilov2017sgdr}. The backbone learning rate is $10^{-4}$. The classifier head is initialized at random, trained jointly with the backbone, and uses a learning rate of $10^{-3}$ (ten times the backbone rate). Each model is trained for $20$ epochs with batch size $32$, and the final-epoch checkpoint is used for all subsequent merging and evaluation. Training images are augmented with a random resized crop of size $224$ at scale $0.6$ to $1.0$, followed by a random horizontal flip and ImageNet normalization. Evaluation images are resized to $256$ pixels on the short side, center-cropped to $224$, and ImageNet-normalized. The random seed is fixed at $42$ across all fine-tuning runs.

Merging follows the standard task-arithmetic recipe of \citet{ilharco2023editing}, which averages fine-tuned weights in the spirit of model soups~\citep{wortsman2022model}. For pairwise merging, two fine-tuned checkpoints of the same backbone are combined by averaging their backbone weights, $\theta_{\text{merged}} = 0.5 \, \theta_a + 0.5 \, \theta_b$. The classifier head from each task is restored at evaluation time, so that the merged backbone is paired with the original task-specific head when scoring that task. For the multi-task task-arithmetic extension reported in \Cref{app:merge_extras}, the merged backbone is $\theta_{\text{base}} + \alpha \sum_i (\theta_i - \theta_{\text{base}})$ across all eight tasks, with $\alpha \in \{0.3, 0.5\}$, also following \citet{ilharco2023editing}.

For each task we report top-$1$ accuracy on the standard test split. The merge quality of a task pair is the mean retention $\tfrac{1}{2}(\text{acc}_a / \text{acc}^{\text{solo}}_a + \text{acc}_b / \text{acc}^{\text{solo}}_b)$, where $\text{acc}^{\text{solo}}$ is the accuracy of the corresponding fine-tuned baseline before merging. We evaluate every within-model task pair across all $14$ backbones, giving $14 \times \binom{8}{2} = 392$ pairwise merges.

\subsection{Hyperparameters}
\label{app:setup_hyperparameters}

\Cref{tab:csa_hp} summarizes the hyperparameters of the cross-modal alignment pipeline and \Cref{tab:merge_hp} summarizes those of the task-arithmetic fine-tuning and merging pipeline. Together they are sufficient to reproduce both pipelines without reading the source code.

\begin{table}[H]
\centering
\small
\caption{Cross-modal alignment hyperparameters.}
\label{tab:csa_hp}
\begin{tabular}{ll}
\toprule
Component & Value \\
\midrule
Adapter & Normalized canonical correlation analysis (CSA, \citealp{li2025csa}) \\
Encoders & Frozen \\
Canonical dimension & $512$ \\
Equal-weight constraint & No \\
Regularization & None \\
Layer-selection probe size & $5{,}000$ paired image-caption embeddings \\
Layer-selection metric & Mutual $k$-NN overlap, $k = 10$ \\
Training data & MS-COCO 2014 captions (training split) \\
Embedding augmentation & None \\
Evaluation batch size & $256$ \\
Random seed & $42$ \\
\bottomrule
\end{tabular}
\end{table}

\begin{table}[H]
\centering
\small
\caption{Task-arithmetic fine-tuning and merging hyperparameters.}
\label{tab:merge_hp}
\begin{tabular}{ll}
\toprule
Component & Value \\
\midrule
Pretrained source & Public \texttt{timm} checkpoints~\citep{wightman2019timm} \\
Fine-tuning method & Full fine-tuning \\
Optimizer & AdamW~\citep{loshchilov2019decoupled} \\
Weight decay & $0.01$ \\
Backbone learning rate & $10^{-4}$ \\
Classifier-head learning rate & $10^{-3}$ \\
Scheduler & Cosine annealing \\
Loss & Cross-entropy \\
Epochs & $20$ \\
Batch size & $32$ \\
Train augmentation & Random resized crop, random horizontal flip \\
Eval augmentation & Resize to $256$ on the short side, center crop to $224$ \\
Checkpoint selection & Final epoch \\
Pairwise merge coefficient & $\alpha = 0.5$ \\
Multi-task merge coefficients & $\alpha \in \{0.3, 0.5\}$ \\
Random seed & $42$ \\
\bottomrule
\end{tabular}
\end{table}

\subsection{Compute resources}
\label{app:setup_compute}

All experiments run on a cluster of 8 NVIDIA GeForce RTX 3090 GPUs (24~GB each), with each individual training run scheduled to a single GPU. Per-run memory is below 4~GB, and per-run wall time is at most 30 minutes. The total compute budget across all reported experiments is approximately 100 GPU-hours, including the 72 cross-modal alignment runs, the 112 fine-tuning runs that produce the merging checkpoints, and the forward passes used to compute $k$-NN overlap and CKA across the model library.

\subsection{Datasets and licenses}
\label{app:setup_licenses}

\Cref{tab:datasets_licenses} lists the datasets used in the paper together with their licenses. All datasets are publicly available and used in accordance with their respective terms.

\begin{table}[H]
\centering
\small
\caption{Datasets used in this paper with their licenses.}
\label{tab:datasets_licenses}
\begin{tabular}{p{0.25\linewidth} p{0.32\linewidth} p{0.28\linewidth}}
\toprule
Dataset & Used for & License \\
\midrule
WIT-1024 \citep{srinivasan2021wit} & probe set, $k$-NN overlap & CC BY-SA 3.0 \\
MS-COCO \citep{lin2014microsoft} & alignment training, retrieval & CC BY 4.0 \\
Flickr30k \citep{plummer2015flickr30k} & retrieval evaluation & Flickr terms of use \\
\midrule
\multicolumn{3}{l}{\textit{Task-arithmetic merging benchmark~\citep{ilharco2023editing}}} \\
MNIST & merging fine-tuning & CC BY-SA 3.0 \\
SVHN & merging fine-tuning & non-commercial research use \\
GTSRB & merging fine-tuning & CC0 1.0 \\
DTD & merging fine-tuning & research use \\
EuroSAT & merging fine-tuning & MIT \\
RESISC45 & merging fine-tuning & non-commercial research use \\
Stanford Cars & merging fine-tuning & research use \\
SUN397 & merging fine-tuning & research use \\
\bottomrule
\end{tabular}
\end{table}

\subsection{Code availability}
\label{app:setup_code}

The full code base for the convergence, alignment, and merging experiments, together with the configuration files and scripts used to produce every figure in the paper, will be released upon acceptance.

\vspace{-1em}
\section{Broader impact}
\label{app:broader_impact}

\name is a measurement framework for selecting model pairs before alignment or merging. It introduces no new pretrained models, releases no scraped data, and does not change the capabilities of any existing system. Its intended use is to reduce the wasted compute spent on transferring incompatible model pairs, which we view as a net positive for accessibility and energy use. We do not foresee direct negative societal impacts beyond those already associated with the underlying pretrained models, which are released by their original authors under their stated licenses.

\section{Why local: empirical evidence}
\label{app:hierarchy}

This experiment motivates the choice of local $k$-NN overlap rather than global CKA as the cheap-to-compute proxy used in \Cref{sec:practical}. Local neighborhood recovery is a strictly weaker requirement than global distance recovery, so local metrics can track alignment trends in regimes where global metrics flatten or diverge. For each fixed reference model, we correlate \emph{calibrated} alignment with the partner model's scale (parameter count).

\begin{figure}[H]
\centering
\includegraphics[width=\linewidth]{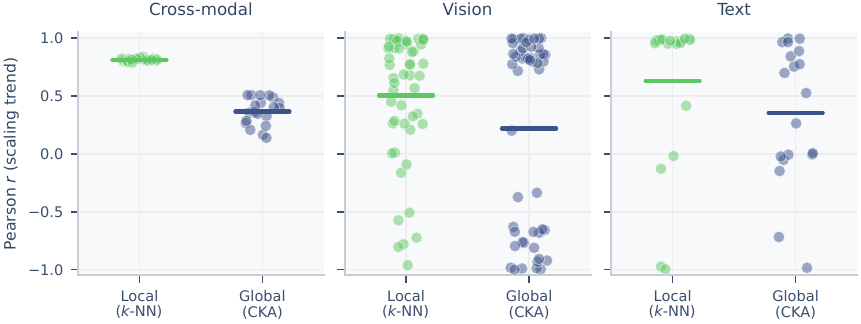}
\caption{Pearson $r$ between calibrated alignment and model scale ($k=10$), for each fixed reference model. Left: for cross-modal, local $k$-NN (green) shows positive scaling (mean $r \approx 0.64$), while global CKA (blue) is weaker after calibration (mean $r \approx 0.37$). Center: within-vision. Right: within-text. The local-global gap motivates measuring local $k$-NN overlap rather than global CKA.}
\label{fig:hierarchy}
\end{figure}

\Cref{fig:hierarchy} shows the local-global gap. Cross-modal local $k$-NN correlations concentrate near $r \approx 0.64$, while global CKA correlations are weaker and more dispersed (mean $r \approx 0.37$), so local alignment captures a stronger scaling trend than global alignment.

\section{Distribution identities for local overlap}
\label{app:distribution_identities}

To simplify notation, throughout this section we suppress the time index $t$ and write all neighborhood distributions, overlaps, and $k$-NN sets without an explicit time dependence.
\subsection{Overlap-TV identities}

The bound in \Cref{thm:local_convergence} relies on two exact identities between $k$-NN overlap and the total variation distance of the corresponding neighborhood distributions.

\begin{proposition}[Overlap-TV equivalence]
\label{prop:tv_overlap}
For the hard $k$-NN distributions defined in~\Cref{eq:hard_distribution},
\begin{equation}
\TV\bigl(p_X(\cdot \mid i),\, p_Y(\cdot \mid i)\bigr) = 1 - o_i.
\end{equation}
\end{proposition}

\begin{proposition}[Smoothed hard-distribution identity]
\label{prop:smoothed_tv}
For each sample $i$ and every $\varepsilon \in (0,1)$,
\begin{equation}
\TV\bigl(q_X(\cdot \mid i),\, q_Y(\cdot \mid i)\bigr)
= (1-\varepsilon)\,(1 - o_i).
\end{equation}
\end{proposition}

\subsection{Proof of \Cref{prop:tv_overlap}}

Fix a sample $i$. By definition, the hard $k$-NN distribution $p_X(\cdot \mid i)$ assigns mass $1/k$ to each element of $\kNN_X(i)$ and mass $0$ to every other index. The same is true for $p_Y(\cdot \mid i)$ with $\kNN_Y(i)$ in place of $\kNN_X(i)$. We use the identity
\begin{equation}
\TV(p,q) = 1 - \sum_j \min\{p(j),q(j)\}
\end{equation}
for probability measures on a finite space. Therefore
\begin{equation}
\TV\bigl(p_X(\cdot \mid i),p_Y(\cdot \mid i)\bigr)
= 1 - \sum_j \min\bigl\{p_X(j \mid i),p_Y(j \mid i)\bigr\}.
\end{equation}
If $j \in \kNN_X(i) \cap \kNN_Y(i)$, both distributions assign mass $1/k$ to $j$, so the minimum is $1/k$. Otherwise at least one of the two distributions assigns mass $0$ to $j$, so the minimum is $0$. The sum therefore contains exactly $|\kNN_X(i) \cap \kNN_Y(i)|$ nonzero terms, each equal to $1/k$, giving
\begin{equation}
\sum_j \min\bigl\{p_X(j \mid i),p_Y(j \mid i)\bigr\}
= \frac{|\kNN_X(i) \cap \kNN_Y(i)|}{k}
= o_i,
\end{equation}
and hence $\TV(p_X,p_Y) = 1 - o_i$. \qed

\subsection{Proof of \Cref{prop:smoothed_tv}}

Fix a sample $i$, and let $\nu(j \mid i) := \frac{1}{N-1}\mathbbm{1}\{j \neq i\}$ denote the uniform distribution over $V \setminus \{i\}$. By definition,
\begin{equation}
q_X(\cdot \mid i)
= (1-\varepsilon)p_X(\cdot \mid i) + \varepsilon\,\nu(\cdot \mid i),
\qquad
q_Y(\cdot \mid i)
= (1-\varepsilon)p_Y(\cdot \mid i) + \varepsilon\,\nu(\cdot \mid i).
\end{equation}
Subtracting,
\begin{equation}
q_X(\cdot \mid i) - q_Y(\cdot \mid i)
= (1-\varepsilon)\bigl(p_X(\cdot \mid i) - p_Y(\cdot \mid i)\bigr).
\end{equation}
Using $\TV(p,q) = \tfrac{1}{2}\sum_j |p(j)-q(j)|$,
\begin{align}
\TV\bigl(q_X(\cdot \mid i), q_Y(\cdot \mid i)\bigr)
&= \tfrac{1}{2}\sum_j (1-\varepsilon)\,\bigl|p_X(j \mid i)-p_Y(j \mid i)\bigr| \\
&= (1-\varepsilon)\,\TV\bigl(p_X(\cdot \mid i), p_Y(\cdot \mid i)\bigr).
\end{align}
\Cref{prop:tv_overlap} identifies the last total variation as $1-o_i$, so $\TV(q_X, q_Y) = (1-\varepsilon)(1-o_i)$. \qed

\subsection{Why we $\varepsilon$-smooth hard distributions}
\label{app:soft}

\Cref{thm:local_convergence} uses KL divergences, which require full-support distributions. The $\varepsilon$-smoothing in \Cref{eq:soft_distribution} mixes the hard distribution with the uniform background and yields a full-support conditional, while \Cref{prop:smoothed_tv} preserves the hard-overlap semantics exactly up to the factor $(1-\varepsilon)$. The smoothing parameter $\varepsilon$ is a fixed analytical regularizer required only to make KL divergences finite. In practice we take $\varepsilon$ small (\textit{e.g.}, $\varepsilon = 1/N$), which keeps the prefactor $1/(1-\varepsilon)$ in \Cref{thm:local_convergence} close to $1$ at the cost of larger KL terms when supports of the learned and target distributions disagree. A natural alternative is to replace the hard distribution with a distance-based soft distribution (\textit{e.g.}, a Gaussian over embedding distances). We avoid this choice because such distributions implicitly assume that local metric distances are meaningful objects of comparison across views. The empirical object of interest, hard $k$-NN overlap, depends only on neighbor identity and not on the distances to those neighbors. The $\varepsilon$-smoothing preserves this property by adding uniform background mass to all non-self indices regardless of their distances, while providing the full support required for KL divergences to be finite.

\section{Pair-specific ease: per-sample start-condition overlap}
\label{app:ease}
\subsection{Dynamical assumptions and proof variables}
\label{app:assumptions}

The finite-time convergence results in \Cref{sec:overlap} rest on three explicit assumptions. Let $(\mathcal F_t)_{t \ge 0}$ be the filtration generated by the training randomness, and let $q_W^{(t)}(\cdot \mid i)$ denote the smoothed distribution of encoder $W \in \{X, Y\}$ after $t$ training steps. We define two per-step total variation quantities:
\begin{equation}
\label{eq:proof_variables}
m_i^{(t)} := \min_{W \in \{X, Y\}} \TV\bigl(q_W^{(t)}(\cdot \mid i),\, q_W^{\star}(\cdot \mid i)\bigr), \qquad \delta_i^{(t)} := \TV\bigl(q_X^{(t)}(\cdot \mid i),\, q_Y^{(t)}(\cdot \mid i)\bigr),
\end{equation}
where $m_i^{(t)}$ is the smaller of the two encoder-to-target distances and $\delta_i^{(t)}$ is the inter-encoder gap. \Cref{prop:smoothed_tv} gives the identity
\begin{equation}
\label{eq:delta0_overlap}
\delta_i^{(0)} = (1-\varepsilon)\bigl(1 - o_i^{(0)}\bigr) \quad \text{at } t=0,
\end{equation}
so the start-condition overlap and the initial inter-encoder gap describe the same starting point.

\paragraph{(A1) Conditional reachability.}
There exist $T > 0$, $\xi > 0$, and $p_T > 0$ such that for every $\delta_i^{(0)}$ in the support of the start condition, $\mathbb{P}\bigl(m_i^{(T)} \le \xi \mid \delta_i^{(0)}\bigr) \ge p_T$. This is a finite-time reachability condition: it does not require that both encoders reach the target, only that one does. The assumption is satisfied whenever the training procedure has a nonzero probability of making sufficient progress within $T$ steps, a condition that holds for stochastic gradient methods on objectives with bounded noise, provided the target is reachable from the initialization.

\paragraph{(A2) Finite-time decay.}
There exists a nonincreasing function $\lambda:\mathbb{N}_0\to[0,\infty)$ with $\lambda(t)\to 0$ as $t\to\infty$ such that, for every $t \ge 0$,
\begin{equation}
\mathbb{E}\bigl[\delta_i^{(t)} \mid \delta_i^{(0)}\bigr]
\le \lambda(t)\,\delta_i^{(0)} .
\end{equation}
This is a stability condition on the training dynamics: the expected inter-encoder gap after $t$ steps is controlled by its initial value. The assumption is problem-dependent and strongest when both encoders share training dynamics (\textit{e.g.}, the same objective, same data, similar architectures). A2 controls discrepancy, not convergence: it bounds how far apart the two encoders can be at time $t$, not whether either has reached the target. A relaxation allowing an additive noise floor is developed in \Cref{app:noisy_contraction}.

\paragraph{(A3) Basin condition.}
There exists $\eta > 0$ such that $\bigl\{m_i^{(T)} \le \xi,\ \delta_i^{(T)} \le \eta\bigr\} \subseteq \cC_i$, where $\cC_i := \{\lim_{t \to \infty} \max_W \TV(q_W^{(t)}, q_W^{\star}) = 0\}$ is the eventual-convergence event. This is a deterministic inclusion: if at time $T$ at least one encoder is near the target and the two encoders are mutually close, then both eventually converge. A3 is the only assumption used to transfer the finite-time bound (\Cref{thm:joint_proximity}) to eventual convergence (\Cref{cor:overlap_convergence}), and is not used in the finite-time theorem itself. The assumption is natural when the shared target is a stable fixed point of the training dynamics.

\paragraph{Separation of concerns.}
The two-step structure (first \Cref{thm:joint_proximity} under A1--A2, then \Cref{cor:overlap_convergence} under A1--A3) is intentional. The finite-time theorem is the primary result and gives a concrete bound that depends on the start-condition overlap. The corollary adds A3 to convert this into an asymptotic statement, ensuring that the finite-time bound stands on its own when the basin condition is difficult to verify.

\paragraph{Adapter dynamics in cross-modal alignment.}
In applications where the encoders are frozen and only an adapter is trained (\textit{e.g.}, cross-modal alignment), A1 and A2 apply to the adapter-induced dynamics on $q_W^{(t)}$: the distributions evolve through the adapter while the underlying encoders remain fixed. The same theorem then bounds the probability that the adapter brings the two distributions close to each other and to their respective targets within $T$ adapter-training steps.

\subsection{Finite-time control of the inter-encoder gap}

\begin{lemma}[Finite-time control]
\label{lem:finite_time_delta}
Under A2, $\mathbb{E}\bigl[\delta_i^{(T)} \mid \delta_i^{(0)}\bigr] \le \lambda(T)\,\delta_i^{(0)}$.
\end{lemma}

\noindent\emph{Proof.}
This is exactly Assumption A2. \qed

\subsection{Two-form statement and proof of \Cref{thm:joint_proximity}}
\label{app:proof_overlap}

Under A1--A2, for every $\eta > 0$, the finite-time bound takes two equivalent forms:
\begin{equation}
\mathbb{P}\bigl(m_i^{(T)} \le \xi,\ \delta_i^{(T)} \le \eta \mid \delta_i^{(0)}\bigr)
\ge \left[p_T - \frac{\lambda(T)\,\delta_i^{(0)}}{\eta}\right]_+
\qquad\text{(discrepancy form)}
\end{equation}
and equivalently, by the affine bijection \Cref{eq:delta0_overlap} between $\delta_i^{(0)}$ and $o_i^{(0)}$,
\begin{equation}
\mathbb{P}\bigl(m_i^{(T)} \le \xi,\ \delta_i^{(T)} \le \eta \mid o_i^{(0)}\bigr)
\ge \left[p_T - \frac{\lambda(T)(1-\varepsilon)\bigl(1-o_i^{(0)}\bigr)}{\eta}\right]_+
\qquad\text{(overlap form).}
\end{equation}
The discrepancy form is the natural output of the proof, and the overlap form is the version most directly connected to the empirical quantity $o_i$.

\noindent\emph{Proof.}
Fix a sample $i$ and define $A := \{m_i^{(T)} \le \xi\}$ and $B := \{\delta_i^{(T)} \le \eta\}$.

\emph{Step 1: separate the two requirements.}
Since $A = (A \cap B) \cup (A \cap B^c)$ with the two parts disjoint, conditional additivity gives $\mathbb P(A \mid \delta_i^{(0)}) = \mathbb P(A \cap B \mid \delta_i^{(0)}) + \mathbb P(A \cap B^c \mid \delta_i^{(0)})$. Rearranging and using $A \cap B^c \subseteq B^c$,
\begin{equation}
\mathbb P(A \cap B \mid \delta_i^{(0)})
\ge \mathbb P(A \mid \delta_i^{(0)}) - \mathbb P(B^c \mid \delta_i^{(0)}).
\label{eq:jointproof_split}
\end{equation}

\emph{Step 2: lower-bound the reachability term.}
Assumption A1 gives $\mathbb P(A \mid \delta_i^{(0)}) \ge p_T$.

\emph{Step 3: upper-bound the discrepancy-failure term.}
Markov's inequality applied conditionally on $\delta_i^{(0)}$ gives $\mathbb P(\delta_i^{(T)} > \eta \mid \delta_i^{(0)}) \le \mathbb E[\delta_i^{(T)} \mid \delta_i^{(0)}]/\eta$. \Cref{lem:finite_time_delta} bounds the numerator by $\lambda(T)\,\delta_i^{(0)}$, so $\mathbb P(B^c \mid \delta_i^{(0)}) \le \lambda(T)\,\delta_i^{(0)} / \eta$.

\emph{Step 4: combine.}
Inserting the two bounds into ~\Cref{eq:jointproof_split} and taking the positive part (since the left-hand side is a probability) gives the discrepancy form. Substituting ~\Cref{eq:delta0_overlap} gives the overlap form. \qed

\medskip
\noindent\emph{Informativeness regime.} The bound is informative when $\lambda(T)(1-\varepsilon)\bigl(1-o_i^{(0)}\bigr) < p_T\,\eta$. Outside this regime the right-hand side is zero and the theorem yields no nontrivial guarantee. In particular, for fixed $\eta > 0$ the bound becomes informative for sufficiently large $T$ whenever $\lambda(T)\to 0$, and the threshold on $T$ shrinks with the start-condition overlap $o_i^{(0)}$.

\subsection{Convergence under basin condition}
\label{app:corollary}
\begin{corollary}[Convergence under basin condition]
\label{cor:overlap_convergence}
Under A1--A3,
\begin{equation}
\begin{aligned}
\mathbb{P}\bigl(\cC_i \mid \delta_i^{(0)}\bigr)
&\ge \left[p_T - \frac{\lambda(T)\,\delta_i^{(0)}}{\eta}\right]_+ \quad\text{and equivalently}\\
\mathbb{P}\bigl(\cC_i \mid o_i^{(0)}\bigr)
&\ge \left[p_T - \frac{\lambda(T)(1-\varepsilon)\bigl(1-o_i^{(0)}\bigr)}{\eta}\right]_+.
\end{aligned}
\end{equation}
\end{corollary}

\noindent\emph{Proof.}
Assumption A3 gives $\{m_i^{(T)} \le \xi,\ \delta_i^{(T)} \le \eta\} \subseteq \cC_i$, so monotonicity of probability gives $\mathbb P(\cC_i \mid \delta_i^{(0)}) \ge \mathbb P(m_i^{(T)} \le \xi,\ \delta_i^{(T)} \le \eta \mid \delta_i^{(0)})$. Substituting the bound from \Cref{thm:joint_proximity} (discrepancy form) yields the first inequality, and~\Cref{eq:delta0_overlap} converts it into the overlap form. \qed

\subsection{Extension to noisy finite-time decay}
\label{app:noisy_contraction}

A2 requires exact finite-time decay:
$\mathbb{E}[\delta_i^{(t)} \mid \delta_i^{(0)}] \le \lambda(t)\,\delta_i^{(0)}$,
where $\lambda(t)$ is nonincreasing and tends to zero. A more realistic model of SGD allows an additive noise floor.

\noindent\textbf{(A2$'$) Noisy finite-time decay.} There exist functions
$\lambda:\mathbb{N}_0\to[0,\infty)$ and $b:\mathbb{N}_0\to[0,\infty)$, with $\lambda(t)\to 0$, such that for every $t \ge 0$,
\begin{equation}
\mathbb{E}\bigl[\delta_i^{(t)} \mid \delta_i^{(0)}\bigr]
\le \lambda(t)\,\delta_i^{(0)} + b(t).
\end{equation}

\medskip
When $b(t)=0$, A2$'$ reduces to A2. We refer to $b(t)$ as the finite-time noise floor.

\begin{lemma}[Finite-time control under noisy decay]
\label{lem:noisy_delta}
Under A2$'$,
\begin{equation}
\mathbb{E}\bigl[\delta_i^{(T)} \mid \delta_i^{(0)}\bigr]
\le \lambda(T)\,\delta_i^{(0)} + b(T).
\end{equation}
\end{lemma}

\noindent\emph{Proof.}
This is exactly Assumption A2$'$. \qed

\begin{theorem}[Finite-time bound under noisy decay]
\label{thm:noisy_proximity}
Under A1 and A2$'$, for every $\eta > 0$,
\begin{equation}
\mathbb{P}\bigl(m_i^{(T)} \le \xi,\ \delta_i^{(T)} \le \eta \mid o_i^{(0)}\bigr)
\ge
\left[
p_T -
\frac{\lambda(T)(1-\varepsilon)\bigl(1-o_i^{(0)}\bigr) + b(T)}{\eta}
\right]_+.
\end{equation}
When $b(T)=0$, this reduces to \Cref{thm:joint_proximity}.
\end{theorem}

\noindent\emph{Proof.}
The proof follows that of \Cref{thm:joint_proximity}, substituting \Cref{lem:noisy_delta} for \Cref{lem:finite_time_delta} in the Markov step:
\begin{equation}
\mathbb{P}\bigl(\delta_i^{(T)} > \eta \mid \delta_i^{(0)}\bigr)
\le
\frac{\mathbb{E}[\delta_i^{(T)} \mid \delta_i^{(0)}]}{\eta}
\le
\frac{\lambda(T)\,\delta_i^{(0)} + b(T)}{\eta}.
\end{equation}
Using $\delta_i^{(0)}=(1-\varepsilon)\bigl(1-o_i^{(0)}\bigr)$ gives the stated overlap form. \qed

The noise floor $b(T)$ enters additively: even with perfect start-condition overlap ($o_i^{(0)} = 1$), the bound is $[p_T - b(T)/\eta]_+$ rather than $[p_T]_+$. For the bound to be nontrivial, $\eta$ must exceed $b(T) / p_T$, reflecting that noisy dynamics cannot guarantee arbitrarily precise convergence.

\section{Paradigm-level ceiling on overlap}
\label{app:ceiling}

To simplify notation, throughout this section we suppress the time index $t$ and write all neighborhood distributions, overlaps, and $k$-NN sets without an explicit time dependence.

\subsection{Proof of \Cref{thm:local_convergence}}
\label{app:proof_local_convergence}

Fix a sample $i$ and write $q_X^i := q_X(\cdot \mid i)$, $q_Y^i := q_Y(\cdot \mid i)$, $q_X^{\star i} := q_X^{\star}(\cdot \mid i)$, and $q_Y^{\star i} := q_Y^{\star}(\cdot \mid i)$.

\paragraph{Step 1: rewrite hard overlap in terms of smoothed total variation.}
\Cref{prop:smoothed_tv} gives
\begin{equation}
\TV(q_X^i,q_Y^i) = (1-\varepsilon)\bigl(1-o_i\bigr),
\qquad\text{equivalently}\qquad
1-o_i = \frac{1}{1-\varepsilon}\,\TV(q_X^i,q_Y^i).
\label{eq:proof_local_step1}
\end{equation}

\paragraph{Step 2: reverse triangle inequality through the two targets.}
Two applications of the reverse triangle inequality for total variation give
\begin{equation}
\TV(q_X^i,q_Y^i)
\ge \TV(q_X^{\star i},q_Y^{\star i}) - \TV(q_X^i,q_X^{\star i})  - \TV(q_Y^{\star i},q_Y^i).
\label{eq:proof_local_step2}
\end{equation}
The first term is $\gamma_i$ by definition~\Cref{eq:target_discrepancy}.

\paragraph{Step 3: bound the outer terms by KL.}
Pinsker's inequality, $\TV(p,q) \le \sqrt{\tfrac{1}{2}\KL(p\|q)}$, applied with $(p,q)=(q_X^{\star i},q_X^i)$ and $(q_Y^{\star i},q_Y^i)$ gives
\begin{equation}
\TV(q_X^{\star i},q_X^i) \le \sqrt{\tfrac{1}{2}\KL(q_X^{\star i}\|q_X^i)},
\qquad
\TV(q_Y^{\star i},q_Y^i) \le \sqrt{\tfrac{1}{2}\KL(q_Y^{\star i}\|q_Y^i)}.
\label{eq:proof_local_step3}
\end{equation}
Substituting \Cref{eq:proof_local_step3} into \Cref{eq:proof_local_step2} and then into \Cref{eq:proof_local_step1} yields the claimed bound.

\paragraph{Sequence statement.}
For a sequence of learned distributions indexed by $t$ such that both KL terms vanish, $\limsup_t\, o_{i}^{(t)} \le 1 - \gamma_i/(1-\varepsilon)$.When $q_X^{\star} = q_Y^{\star}$, $\gamma_i = 0$ and $o_{i}^{(t)} \to 1$. \qed

\subsection{A natural shared target}
\label{app:shared_target}

\Cref{thm:local_convergence} holds for any pair of full-support targets. One natural setting where $\gamma_i = 0$ is when both encoders observe paired views of the same underlying content (\textit{e.g.}, an image and its caption): a conditional affinity distribution
\begin{equation}
\label{eq:canonical_distribution}
\pstar(j \mid i) \propto \kappa(i, j), \qquad j \neq i,
\end{equation}
serves as a shared target, where $\kappa(i, j)$ measures co-occurrence or affinity between paired contents and is therefore independent of the view. For instance, when pairs arise under an unobserved context variable, $\kappa(i, j)$ is proportional to the expected co-occurrence of the two contents under that context, which is compatible with contrastive views of representation learning~\citep{oord2018representation}. Smoothing this candidate via
\begin{equation}
\label{eq:qstar_smooth}
q_\star^{(\varepsilon)}(j \mid i) := (1-\varepsilon)\,\pstar(j \mid i) + \frac{\varepsilon}{N-1}\,\mathbbm{1}\{j \neq i\}
\end{equation}
yields a full-support target suitable for \Cref{thm:local_convergence}. We do not require this specific choice: different objectives or modalities may induce different targets ($\gamma_i > 0$), and the bound remains informative whenever $\gamma_i$ is small.

\section{Additional empirical material}

\subsection{Initialization structure}
\label{app:init_structure}

The unimodal panels of \Cref{fig:init_structure} use the initial random $k$-NN overlap $o_i^{(0)}$ as a start-condition proxy for \Cref{thm:joint_proximity}. For the proxy to be well-posed, three things must hold at initialization: (i) two copies of the same randomly initialized model should agree on their local geometry (random seeds of the same architecture give a meaningful family-level $\delta_i^{(0)}$), (ii) models from different families should disagree more (architecture modulates $\delta_i^{(0)}$), and (iii) training must actually change local geometry (otherwise the final-vs-initial correlation would be a tautology). \Cref{fig:init_structure_vision,fig:init_structure_text} verify all three for vision and text respectively.

\begin{figure}[t]
\centering
\begin{subfigure}{0.48\linewidth}
    \centering
    \includegraphics[width=\linewidth]{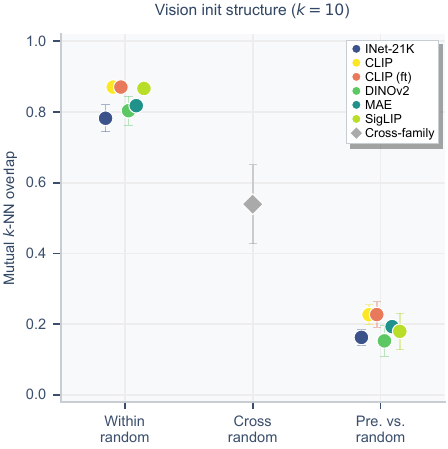}
    \caption{Vision initialization structure at $k=10$. Within-model random overlap is high and concentrated, cross-family random overlap is lower, and pretrained-vs-random overlap is near zero.}
    \label{fig:init_structure_vision}
\end{subfigure}
\hfill
\begin{subfigure}{0.48\linewidth}
    \centering
    \includegraphics[width=\linewidth]{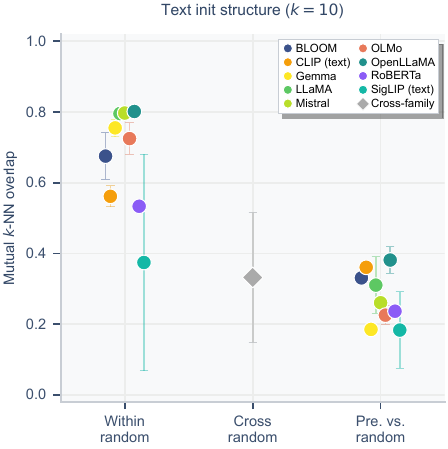}
    \caption{Text initialization structure at $k=10$. The same three-zone pattern as for vision: high within-family random overlap, intermediate cross-family, near-zero pretrained-vs-random.}
    \label{fig:init_structure_text}
\end{subfigure}
\caption{Initialization structure underlying the unimodal panels of \Cref{fig:init_structure}. Both modalities show structured random geometry within architecture families, lower cross-family random overlap, and a large drop between pretrained and random representations.}
\label{fig:init_structure_appendix}
\end{figure}

Within-model random overlap in the vision panel sits between $0.72$ and $0.89$ for every architecture family (mean $0.83$), indicating that two random seeds of the same architecture produce nearly identical local neighborhood structure. Cross-family random overlap has mean $\approx 0.54$, indicating that architecture choice does shift the initial geometry but less than across paradigms. Pretrained-vs-random overlap drops to $0.09$-$0.28$ (mean $0.19$), confirming that training substantially reshapes the local geometry rather than merely refining what the random network already encoded. Architecture alone therefore parameterizes $\delta_i^{(0)}$, and the training dynamics move every model far from its initialization, so \Cref{thm:joint_proximity} is tested non-trivially in this regime.

The text panel shows the same qualitative three-zone pattern across nine language-model families: within-family random overlap between $0.4$ and $0.8$, cross-family at $\approx 0.35$, and pretrained-vs-random collapsed to $0.2$-$0.4$. The within-family value is lower than for vision and spreads more, reflecting the wider design space of language-model blocks. The pretrained-vs-random drop is nonetheless large in every family, so the init-vs-final correlation in \Cref{fig:init_structure} (top right) reflects genuine learned structure rather than initialization bleed-through.

Taken together, the two panels show that families with similar architectures start with similar local geometry and that training moves every model far from its initialization. This is the condition under which the architectural proxy $\delta_i^{(0)}$ should predict the family-level spread of final overlap, and \Cref{fig:init_structure} (top) confirms that it does.

\subsection{Robustness to neighborhood size}
\label{app:robustness_k_metric}

The main text uses local $k$-NN overlap with $k=10$ throughout the empirical analyses. 
We repeat the combined proxies figure of \Cref{fig:init_structure} at $k \in \{20, 50, 100\}$ to check that the start-condition overlap continues to predict downstream transfer at larger neighborhood sizes. 
The qualitative conclusions remain unchanged: the start-condition overlap continues to predict the pretrained overlap in the unimodal validation and to predict downstream transfer in cross-modal alignment and task-arithmetic merging.

\begin{figure}[t]
\centering
\includegraphics[width=\linewidth]{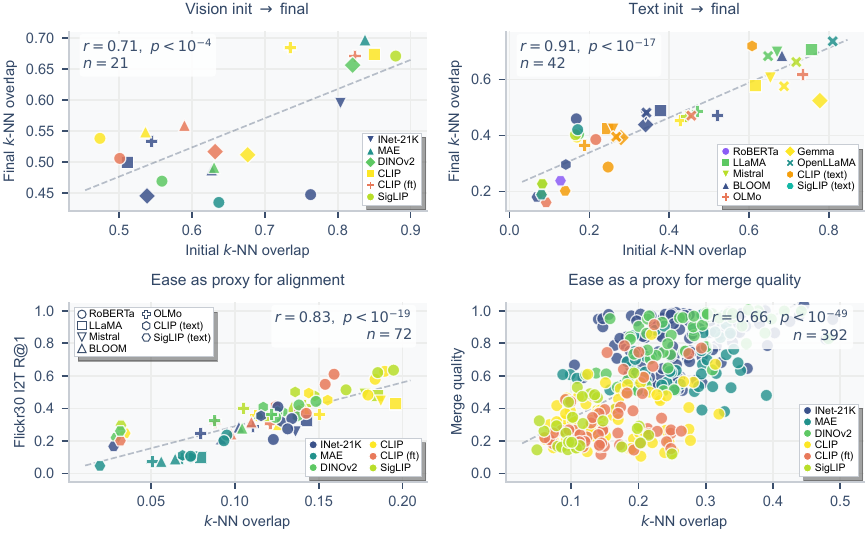}
\caption{Combined proxies figure at $k=20$. Top: unimodal init-to-final validation for vision and text. Bottom left: pretrained $k$-NN overlap predicts Flickr30k I2T R@1 after alignment. Bottom right: pre-merge $k$-NN overlap predicts merge quality. The same qualitative pattern as at $k=10$ is preserved.}
\label{fig:theorem3_k20}
\end{figure}

\begin{figure}[H]
\centering
\includegraphics[width=\linewidth]{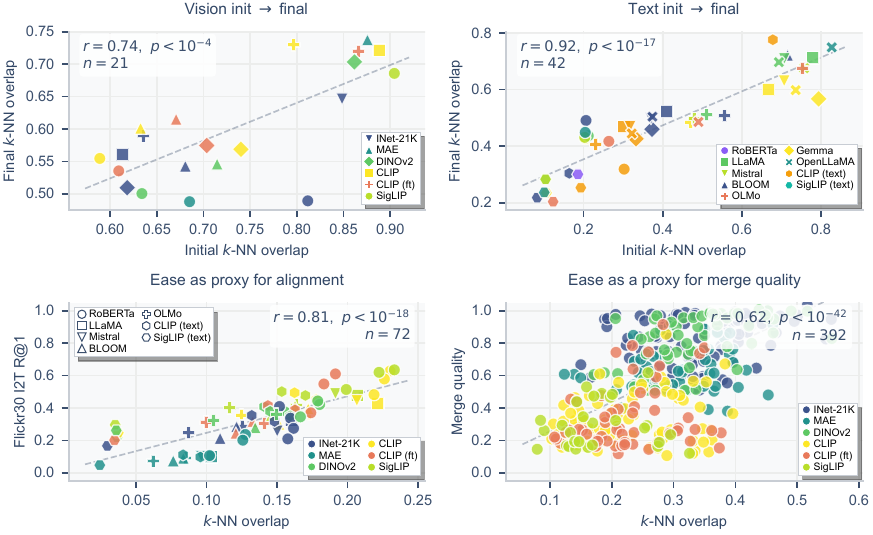}
\caption{Combined proxies figure at $k=50$. The start-condition overlap remains predictive at larger neighborhood size across all four settings.}
\label{fig:theorem3_k50}
\end{figure}

\begin{figure}[H]
\centering
\includegraphics[width=\linewidth]{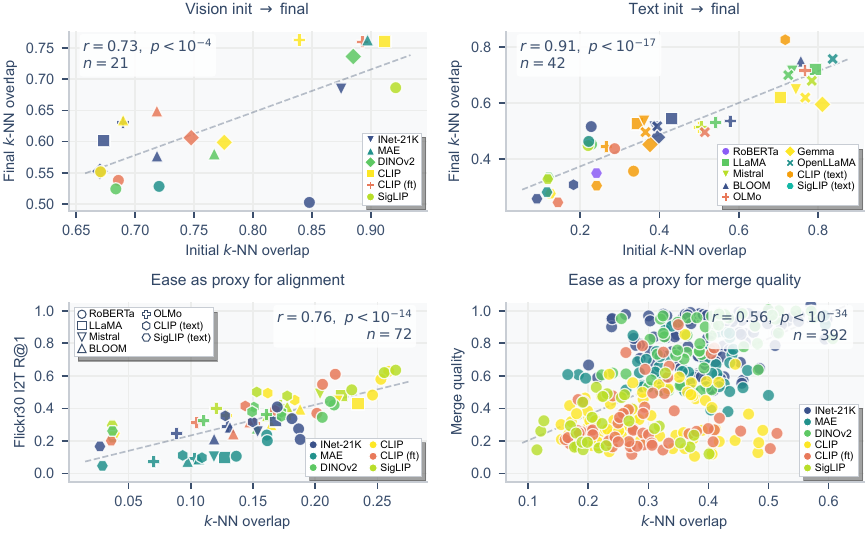}
\caption{Combined proxies figure at $k=100$. Even at a large neighborhood size, the qualitative structure is unchanged.}
\label{fig:theorem3_k100}
\end{figure}

\subsection{Probe-set sensitivity for merging ease}
\label{app:probe_sensitivity}

In the main text, the merging ease score is computed on a 2K-sample probe drawn from MNIST's train split (a task in the standard task-arithmetic benchmark, used as a fixed shared probe across all candidate pairs). 
To test the sensitivity of the diagnostic to this probe choice, we replace MNIST with an 8-task pool: 256 images per task across all eight fine-tuning tasks ($2{,}048$ samples total; Table~\ref{tab:probe_sensitivity}).

\begin{table}[h]
\centering
\small
\caption{Probe-set sensitivity for the merging ease score. Replacing the MNIST probe with an 8-task pool weakens the ease signal (the pool dilutes per-pair specialization across 8 distributions), but the product
predictor remains positively correlated with merge quality. The same direction holds: pairs with higher ease and higher ceiling merge better.}
\vspace{1em}
\begin{tabular}{lccc}
\toprule
Ease probe & $|r|$ ease & $|r|$ ceiling & $|r|$ product \\
\midrule
MNIST (2K samples, main text) & $0.67$ & $0.64$ & $0.70$ \\
8-task pool (256 / task, 2{,}048 samples) & $0.42$ & $0.64$ & $ 0.65$ \\
\bottomrule
\end{tabular}
\label{tab:probe_sensitivity}
\end{table}

\FloatBarrier

\vspace{-1em}
\subsection{Merging: paradigm-level results and multi-task extension}
\label{app:merge_extras}

This section supplements \Cref{sec:practical_merge} with a paradigm-level view of the within-model merge experiments and a full multi-task task-arithmetic extension ($n = 392$ pairs from 14 models, \textit{i.e.}, 9 ViT-B/16, 4 ViT-L/14-16, and 1 ViT-T/16, across 5 paradigms and 8 tasks).

\Cref{fig:merge_combined} shows two complementary views of the same merging experiments. Left: pairwise mean retention ($n=392$ pairs) grouped by paradigm, with merge quality encoded by color and median overlap indicated by horizontal lines. The monotonic increase in median overlap from contrastive VL to self-distillation confirms that the paradigm hierarchy is stable across task pairs. Right: full multi-task task arithmetic, in which all 8 task vectors are summed simultaneously ($\theta_{\mathrm{base}} + \alpha \sum_i \Delta_i$) at $\alpha \in \{0.3, 0.5\}$, and accuracy is averaged across all 8 tasks. The pairwise paradigm ordering is preserved at scale: discriminative and self-distillation paradigms retain the highest accuracy, while contrastive VL paradigms collapse.

This ordering is the inverse of the one observed for cross-modal alignment,
where contrastive vision-language models lead and discriminative or
self-distillation models trail. The inversion is consistent with the ceiling
bound: achievable overlap depends on the match between the pretraining
paradigm and the downstream target, rather than on the paradigm alone. Thus,
the pairwise compatibility ordering is not an artifact of two-task merges,
but remains visible in the harder multi-task setting as well.

\begin{figure}[H]
\centering
\includegraphics[width=\linewidth]{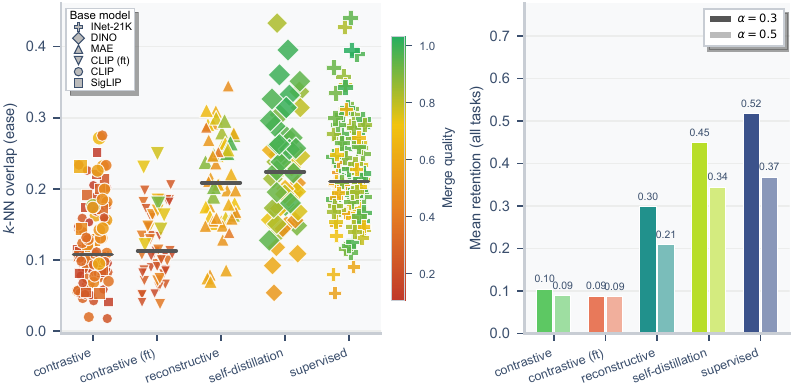}
\caption{Left: pair-specific $k$-NN overlap by paradigm, with median lines. The paradigm ordering is stable across task pairs ($n = 392$ pairwise merges). Color encodes pairwise merge quality (mean retention across the two original tasks). 
Marker shape encodes base model family. Right: full multi-task task arithmetic (all 8 tasks summed jointly), reporting mean accuracy across all 8 tasks. 
The pairwise ordering of \Cref{fig:merge_phase} is preserved at the multi-task scale: discriminative ($0.52/0.37$ multi-task accuracy) and self-distillation ($0.45/0.34$) retain the most, contrastive VL collapses ($0.10/0.09$). 
Note that these multi-task accuracies differ from the pairwise mean retentions reported in \Cref{sec:practical_merge}, since they aggregate across all 8 tasks at once rather than over pairs of tasks.
}
\label{fig:merge_combined}
\end{figure}

\subsection{Regime-level distribution of downstream outcomes}
\label{app:regime_boxplots}

This section gives the full distribution of downstream outcomes per regime, supplementing the inline regime statistics reported in \Cref{sec:practical_align,sec:practical_merge}. Regimes are assigned by splitting the ease and ceiling axes at their respective medians, as defined in the \name diagnostic of \Cref{sec:practical}. The Kruskal-Wallis statistic on each panel tests the null that the four regimes are drawn from the same distribution of the panel's outcome variable; the test is the rank-based, non-parametric analogue of a one-way ANOVA, and we use it because the per-regime sample sizes are unequal and the outcomes are not normally distributed.

\begin{figure}[H]
\centering
\includegraphics[width=\linewidth]{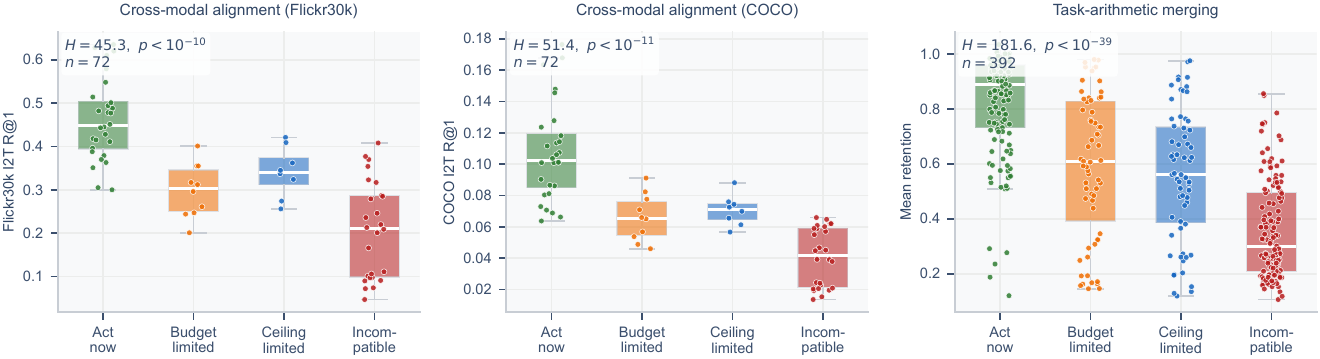}
\caption{Per-regime distribution of downstream outcomes for the three settings. \emph{Left:} Flickr30k I2T R@1 on $n = 72$ vision-language pairs. \emph{Center:} COCO I2T R@1 on the same pairs. \emph{Right:} mean retention on $n = 392$ task-arithmetic merges. Boxes show the interquartile range with the median line; whiskers extend to the non-outlier extremes; individual points are overlaid. Color encodes the four regimes: \emph{act now} (high ease, high ceiling), \emph{budget-limited} (low ease, high ceiling), \emph{ceiling-limited} (high ease, low ceiling), and \emph{incompatible} (low ease, low ceiling). The Kruskal-Wallis $H$ statistic and $p$-value at the top of each panel test the null that the four regimes share the same distribution; the test rejects in all three settings.}
\label{fig:regime_boxplots}
\end{figure}

\subsection{Ranking and budget model selection}
\label{app:ranking_selection}

Because \name is intended to rank candidate model pairs before transfer, we complement the Pearson analysis of \Cref{fig:predictors} with rank correlations and measures of the quality of the induced candidate ordering. All analyses use the same candidate pools and downstream outcomes as in the main experiments.

\subsubsection{Rank correlations}

Table~\ref{tab:rank_correlations} reports Spearman $\rho$ and Kendall $\tau$. COMPASS gives the strongest overall ranking in all settings. The results are nearly identical across the image-to-text (I2T) and text-to-image (T2I) retrieval directions.

\begin{table}[t]
\centering
\small
\caption{Rank correlations with downstream transfer performance. The final column uses global CKA for alignment and task-vector cosine for merging.}
\label{tab:rank_correlations}
\begin{tabular}{lccccc}
\toprule
Outcome &
COMPASS $\rho$ &
COMPASS $\tau$ &
Ease $\rho$ &
Ceiling $\rho$ &
Baseline $\rho$ \\
\midrule
Flickr30k I2T R@1 & 0.91 & 0.73 & 0.87 & 0.77 & 0.63 \\
Flickr30k T2I R@1 & 0.91 & 0.73 & 0.87 & 0.77 & 0.65 \\
COCO I2T R@1      & 0.93 & 0.77 & 0.88 & 0.81 & 0.67 \\
COCO T2I R@1      & 0.93 & 0.77 & 0.88 & 0.81 & 0.66 \\
Merge retention   & 0.70 & 0.50 & 0.66 & 0.61 & 0.06 \\
\bottomrule
\end{tabular}
\end{table}

\FloatBarrier

\subsubsection{Top-$k$ selection, regret, and NDCG}

For evaluation budget $k$, we define top-$k$ performance as the mean downstream outcome of the $k$ highest-ranked candidates. Regret is its relative shortfall from the oracle top-$k$ set. NDCG@$k$ evaluates the quality of the complete ordering within the first $k$ candidates, using downstream performance as relevance and normalizing by the oracle ordering. ``Budget to best'' is the rank assigned to the globally best candidate. Selection results for alignment and merging are reported in
Tables~\ref{tab:topk_flickr} and~\ref{tab:topk_merging}.

\begin{table}[!htbp]
\centering
\scriptsize
\caption{Selection performance on Flickr30k I2T retrieval ($n=72$ candidate pairs). Regret is relative to the oracle top-$k$ set.}
\label{tab:topk_flickr}
\resizebox{\linewidth}{!}{
\begin{tabular}{lcccccccccccc}
\toprule
Predictor &
Top-1 & Regret &
Top-5 & Regret & NDCG@5 &
Top-10 & Regret & NDCG@10 &
Top-20 & Regret & NDCG@20 &
Budget to best \\
\midrule
COMPASS
& 0.635 & 0.0\%
& 0.578 & 5.9\% & 0.957
& 0.530 & 5.7\% & 0.952
& 0.493 & 2.1\% & 0.972
& 1/72 \\

Ease
& 0.428 & 32.6\%
& 0.521 & 15.2\% & 0.819
& 0.530 & 5.7\% & 0.893
& 0.486 & 3.6\% & 0.923
& 2/72 \\

Ceiling
& 0.635 & 0.0\%
& 0.563 & 8.3\% & 0.945
& 0.455 & 19.0\% & 0.864
& 0.439 & 12.8\% & 0.889
& 1/72 \\

Global CKA
& 0.635 & 0.0\%
& 0.589 & 4.2\% & 0.970
& 0.515 & 8.2\% & 0.940
& 0.440 & 12.7\% & 0.901
& 1/72 \\

Paradigm distance
& 0.323 & 49.1\%
& 0.204 & 66.8\% & 0.368
& 0.226 & 59.8\% & 0.406
& 0.245 & 51.3\% & 0.458
& 72/72 \\

Total parameters
& 0.418 & 34.2\%
& 0.352 & 42.6\% & 0.614
& 0.328 & 41.6\% & 0.609
& 0.316 & 37.3\% & 0.629
& 21/72 \\

Random ordering
& 0.334 & 47.4\%
& 0.334 & 45.6\% & -
& 0.334 & 40.5\% & -
& 0.334 & 33.7\% & -
& 36.5 expected \\
\bottomrule
\end{tabular}}
\end{table}

Global CKA happens to rank the best Flickr30k pair first and is slightly stronger at $k=5$, but its ranking deteriorates as the budget increases: at $k=20$, \name has 2.1\% regret compared with 12.7\% for CKA, while also achieving higher NDCG. COCO shows the same pattern: \name/CKA regret is 5.9\%/4.5\% at $k=5$ and 2.1\%/14.4\% at $k=20$.

\begin{table}[!htbp]
\centering
\scriptsize
\caption{Selection performance for task-arithmetic merging ($n=392$ candidate merges).}
\label{tab:topk_merging}
\resizebox{\linewidth}{!}{
\begin{tabular}{lccccccccc}
\toprule
Predictor &
Top-1 & Regret &
Top-10 & Regret & NDCG@10 &
Top-40 & Regret & NDCG@40 &
Budget to best \\
\midrule
COMPASS
& 0.947 & 8.4\%
& 0.927 & 7.1\% & 0.920
& 0.890 & 9.1\% & 0.910
& 5/392 \\

Ease
& 0.947 & 8.4\%
& 0.935 & 6.4\% & 0.924
& 0.874 & 10.7\% & 0.896
& 5/392 \\

Ceiling
& 0.968 & 6.4\%
& 0.926 & 7.3\% & 0.930
& 0.828 & 15.5\% & 0.871
& 162/392 \\

Task-vector cosine
& 0.963 & 6.8\%
& 0.826 & 17.3\% & 0.841
& 0.708 & 27.7\% & 0.749
& 328/392 \\

Random ordering
& 0.585 & 43.4\%
& 0.585 & 41.4\% & -
& 0.585 & 40.3\% & -
& 196.5 expected \\
\bottomrule
\end{tabular}}
\end{table}

Ceiling and task-vector cosine can select a strong first merge but rank the full candidate pool substantially less well. \name finds the globally best merge after evaluating only 5 of the 392 candidates.

\FloatBarrier

\subsubsection{Budget evaluation}

We additionally evaluate the practical setting in which downstream transfer can be performed for only a limited number of candidates. At budget $b$, we take the best downstream outcome among the top $b$ candidates according to each predictor and normalize it by the optimum over the complete candidate pool. Random-order values are exact expectations (Tables~\ref{tab:budget_alignment} and~\ref{tab:budget_merging}).

\begin{table}[!htbp]
\centering
\small
\caption{Budget selection for Flickr30k I2T retrieval. Values are the best outcome found within the evaluation budget, normalized by the pool optimum.}
\label{tab:budget_alignment}
\begin{tabular}{lccccc}
\toprule
Pairs trained & 1 & 2 & 5 & 10 & 20 \\
\midrule
COMPASS     & 1.00 & 1.00 & 1.00 & 1.00 & 1.00 \\
Ease        & 0.67 & 1.00 & 1.00 & 1.00 & 1.00 \\
Global CKA  & 1.00 & 1.00 & 1.00 & 1.00 & 1.00 \\
Random order& 0.53 & 0.66 & 0.79 & 0.88 & 0.95 \\
\bottomrule
\end{tabular}
\end{table}

COCO gives the same qualitative result: \name reaches the pool optimum after one downstream run, compared with random-order expectations of 0.42 at $b=1$ and 0.68 at $b=5$.

\begin{table}[!htbp]
\centering
\small
\caption{Budget selection for task-arithmetic merging.}
\label{tab:budget_merging}
\begin{tabular}{lccccc}
\toprule
Merges performed & 1 & 5 & 10 & 20 & 50 \\
\midrule
COMPASS           & 0.92 & 1.00 & 1.00 & 1.00 & 1.00 \\
Ease              & 0.92 & 1.00 & 1.00 & 1.00 & 1.00 \\
Ceiling           & 0.94 & 0.94 & 0.95 & 0.96 & 0.97 \\
Task-vector cosine& 0.93 & 0.93 & 0.94 & 0.95 & 0.95 \\
Random order      & 0.57 & 0.87 & 0.93 & 0.95 & 0.97 \\
\bottomrule
\end{tabular}
\end{table}

\FloatBarrier

\subsection{Additional baselines}
\label{app:baselines}

We compare \name against three groups of alternatives: cheap metadata predictors, other representation-similarity measures, and learned regressors. Unless stated otherwise, values are Spearman correlations with the same downstream outcomes and over the same candidate sets used for \name.

\subsubsection{Metadata predictors}

Cheap metadata-based predictors are summarized in
Table~\ref{tab:metadata_baselines}.

\begin{table}[!htbp]
\centering
\small
\caption{Comparison with cheap metadata-based predictors.}
\label{tab:metadata_baselines}
\begin{tabular}{lcc}
\toprule
Predictor & Alignment: Flickr I2T & Merging: retention \\
\midrule
COMPASS & \textbf{0.91} & \textbf{0.70} \\
Pretraining paradigm distance & -0.60 & - \\
Architecture self-consistency & 0.37 & - \\
Total parameters & 0.09 & - \\
Random-init $k$-NN overlap & 0.11 & - \\
Mean solo accuracy of two checkpoints & - & 0.43 \\
Task-vector cosine & - & 0.06 \\
\bottomrule
\end{tabular}
\end{table}

The pretraining paradigm provides useful coarse information, but cannot distinguish candidate pairs within the same paradigm. In the within-family analysis of \Cref{app:within_family}, where the vision family is fixed, ease remains predictive with a median Spearman correlation of 0.80. Local $k$-NN overlap therefore captures pair-specific compatibility beyond metadata labels.

\FloatBarrier

\subsubsection{Alternative representation-similarity metrics}

To isolate the effect of the similarity measure itself, we compare metrics computed using identical model inputs, probe sets, and layer-search procedures (Table~\ref{tab:similarity_baselines}).

\begin{table}[!htbp]
\centering
\small
\caption{Alternative representation-similarity metrics computed on identical inputs.}
\label{tab:similarity_baselines}
\begin{tabular}{lcc}
\toprule
Metric & Alignment: Flickr I2T & Merging: retention \\
\midrule
Mutual $k$-NN overlap & \textbf{0.82} & \textbf{0.42} \\
RSA, rank correlation between RDMs & 0.38 & 0.22 \\
Procrustes similarity & 0.45 & 0.37 \\
Linear CKA & 0.34 & 0.25 \\
\bottomrule
\end{tabular}
\end{table}

Local 
$k$-NN overlap has highest predictive performance in both settings, particularly for alignment.

\FloatBarrier

\subsubsection{Learned regressor predictors}

 We additionally fit ridge and gradient-boosted regressors on one-hot family, paradigm, and task labels plus model sizes, and score their out-of-fold
predictions (Table~\ref{tab:learned_regressors}).

\begin{table}[!htbp]
\centering
\small
\caption{Comparison with learned metadata regressors. Unlike COMPASS, these models require previously observed downstream transfer outcomes as labels.}
\label{tab:learned_regressors}
\begin{tabular}{lcc}
\toprule
Evaluation & Alignment & Merging \\
\midrule
COMPASS, nothing fitted & 0.91 & 0.70 \\
Regressor, random 5-fold CV & 0.94 & 0.91 \\
Regressor, leave-one-family/backbone out & 0.60 & 0.81 \\
Regressor, leave-one-paradigm out & 0.46 & -0.54 \\
\bottomrule
\end{tabular}
\end{table}

The regressors can perform stronger under random cross-validation, but require completed transfer runs as supervision, which is the expensive step \name is meant to precede. For example: its merging figure of 0.81 comes from fitting on 364 completed merges in order to rank the remaining 28. It also degrades where a diagnostic is needed, on a model or paradigm it has not seen, and turns negative when a whole paradigm is held out. \name  fits no parameters and is computed before any transfer, so its number is the same in all three rows.

\FloatBarrier

\subsection{Generalization across held-out families, backbones, and paradigms}
\label{app:heldout}

We test whether the \name ranking depends on particular families, backbones, or paradigms. Because \name is parameter-free, we consider two complementary evaluations. In the \emph{rest} analysis, one group is removed, all scores that depend on the the reference library are recomputed without that group, and performance is measured on the remaining candidate pool. In the \emph{held-out only} analysis, the  correlation is evaluated only among candidates belonging to the removed group, isolating pair-specific variation within that group. 
Library-based scores are re-estimated per fold, excluding the held-out language family for alignment or paradigm for merging (Table~\ref{tab:heldout_summary}).

\begin{table}[!htbp]
\centering
\small
\caption{Generalization under held-out groups. The pooled references are $\rho=0.91$ for Flickr30k alignment and $\rho=0.70$ for merging.}
\label{tab:heldout_summary}
\begin{tabular}{lccc}
\toprule
Fold type & Folds & Spearman on rest & Spearman inside held-out group \\
\midrule
Leave-one-vision-family-out
& 6 & 0.85-0.93 & 0.81 median (0.56-0.87) \\
Leave-one-text-family-out
& 7 & 0.89-0.92 & 0.86 median (0.78-0.94) \\
Leave-one-vision-paradigm-out
& 5 & 0.85-0.93 & 0.80 median (0.56-0.87) \\
Leave-one-backbone-out (merging)
& 14 & 0.67-0.71 & 0.28 median (-0.13-0.75) \\
Leave-one-paradigm-out (merging)
& 5 & 0.60-0.72 & 0.27 median (0.20-0.66) \\
\bottomrule
\end{tabular}
\end{table}

Alignment is stable to removing individual model groups(Table~\ref{tab:heldout_alignment}): relative to the pooled correlation of 0.91, removing any vision or text family changes the correlation on the remaining pool by at most 0.06. Held-out-group correlations remain positive across all vision and text families.

\begin{table}[!htbp]
\centering
\small
\caption{Per-fold results for Flickr30k I2T alignment.}
\label{tab:heldout_alignment}
\begin{tabular}{lcc@{\hspace{1.2cm}}lcc}
\toprule
Held-out vision family & Rest & Held out &
Held-out text family & Rest & Held out \\
\midrule
CLIP             & 0.91 & 0.78 & CLIP text   & 0.92 & 0.78 \\
SigLIP           & 0.89 & 0.87 & SigLIP text & 0.89 & 0.86 \\
CLIP fine-tuned  & 0.92 & 0.83 & OLMo        & 0.92 & 0.86 \\
DINOv2           & 0.93 & 0.80 & BLOOM       & 0.91 & 0.90 \\
MAE              & 0.85 & 0.87 & RoBERTa     & 0.90 & 0.94 \\
Classification   & 0.92 & 0.56 & LLaMA       & 0.91 & 0.80 \\
                 &      &      & Mistral     & 0.91 & 0.80 \\
\bottomrule
\end{tabular}
\end{table}

In all seven text folds, ceilings re-estimated from the remaining language models preserve the full-library vision ranking (Spearman $\rho=1.00$).

For merging, the correlation on the remaining candidate pool stays between $0.67$ and $0.71$ when leaving out individual backbones, and between $0.60$ and $0.72$ when leaving out complete paradigms. Correlations measured only within one backbone are lower because $74\%$ of the variance in retention lies between backbones. Conditioning the probe on the particular task pair captures additional residual interference and increases the median within-backbone correlation from $0.28$ to $0.44$ across the 14 backbones, each containing 28 task pairs.
\FloatBarrier

\subsection{Robustness to the alignment procedure}
\label{app:adapter_robustness}

The complete 72-pair alignment experiment in the main text uses CSA. To test whether the inferred compatibility ranking is specific to this alignment method, we additionally evaluate 24 vision-language pairs spanning six vision encoders from all five vision paradigms and four text towers. For each pair, we train Linear, MLP, and LoRA-style adapters using a contrastive InfoNCE objective and compare their downstream retrieval rankings with CSA (Table~\ref{tab:adapter_agreement}).

\begin{table}[!htbp]
\centering
\small
\caption{Pair-ranking agreement across alignment procedures. Values below the diagonal are Spearman correlations for Flickr30k I2T and values above the diagonal are for T2I.}
\label{tab:adapter_agreement}
\begin{tabular}{lcccc}
\toprule
 & CSA & Linear & MLP & LoRA \\
\midrule
CSA    & -   & 0.94 & 0.82 & 0.96 \\
Linear & 0.87 & -   & 0.86 & 0.99 \\
MLP    & 0.76 & 0.84 & -   & 0.85 \\
LoRA   & 0.91 & 0.98 & 0.85 & - \\
\bottomrule
\end{tabular}
\end{table}

CSA rankings agree with LoRA at $\rho=0.91$ for I2T retrieval, while pairs of trained adapters agree with each other at $\rho=0.84$-$0.98$. We next repeat the predictor analysis using the outcome produced by each adapter (Table~\ref{tab:adapter_predictors}).

\begin{table}[!htbp]
\centering
\small
\caption{Spearman correlation with Flickr30k I2T R@1 under different alignment procedures.}
\label{tab:adapter_predictors}
\begin{tabular}{lcccc}
\toprule
Predictor & CSA & Linear & MLP & LoRA \\
\midrule
COMPASS    & \textbf{0.87} & \textbf{0.90} & \textbf{0.73} & \textbf{0.92} \\
Ease       & 0.78 & 0.78 & 0.56 & 0.78 \\
Ceiling    & 0.86 & 0.97 & 0.76 & 0.95 \\
Global CKA & 0.71 & 0.79 & 0.60 & 0.78 \\
\bottomrule
\end{tabular}
\end{table}

\name remains predictive of the downstream ranking produced by every trained adapter and exceeds global CKA in all cases. T2I retrieval gives the same qualitative conclusion. The consistency across CSA, Linear, MLP, and LoRA-style alignment indicates that the reported ranking reflects properties of the model pair rather than an artifact of CSA.

\FloatBarrier

\subsection{Cost and robustness of ceiling estimation}
\label{app:ceiling_robustness}

Ceiling estimation reuses the same cached probe representations required to compute pairwise overlaps and therefore introduces no additional model forward passes. For the $19\times10$ cross-modal candidate grid, each model is run over the probe once and the resulting overlap averages are then inexpensive to compute (Table~\ref{tab:compass_cost}).

\begin{table}[!htbp]
\centering
\small
\caption{Cost of \name scoring compared with exhaustive downstream alignment. 
The wall-clock comparison was benchmarked on a single RTX 2080 Ti.}
\label{tab:compass_cost}
\begin{tabular}{p{0.27\linewidth} p{0.25\linewidth} p{0.40\linewidth}}
\toprule
Procedure & Scaling & Cost for full $19\times10$ grid \\
\midrule
\name, including ceiling
& One probe pass per model
& 29 passes; $\sim$21 min; 12 GB cached features \\
Exhaustive downstream evaluation
& One alignment run per pair
& 190 runs; $\sim$86--102 h; 729 GB cached features \\
\bottomrule
\end{tabular}
\end{table}

The ceiling averages themselves take only seconds once the model features and pairwise overlaps have been cached. To test whether \name requires precise model-level ceiling estimates, we progressively coarsen or perturb the ceiling while keeping ease unchanged (Table~\ref{tab:ceiling_noise}).

\begin{table}[!htbp]
\centering
\small
\caption{Sensitivity to ceiling granularity and estimation noise. Values are Spearman correlations with downstream performance.}
\label{tab:ceiling_noise}
\begin{tabular}{lcc}
\toprule
Ceiling estimate & Alignment: Flickr30k I2T & Merging: retention \\
\midrule
Original estimate & 0.909 & 0.696 \\
One shared value per paradigm & 0.903 & 0.699 \\
Gaussian noise at 25\% of ceiling spread & 0.901 & 0.695 \\
Gaussian noise at 50\% of ceiling spread & 0.883 & 0.695 \\
Gaussian noise at 100\% of ceiling spread & 0.830 & 0.691 \\
Ceiling values shuffled across paradigms & 0.675 & 0.654 \\
Ceiling removed: ease alone & 0.866 & 0.657 \\
\bottomrule
\end{tabular}
\end{table}

These results show that \name is robust to coarse ceiling estimates. Replacing the individual values with one value per paradigm changes the correlation in a negligible way for alignment and for merging. It is also robust to moderate estimation error: with noise equal to 50\% of the ceiling's empirical spread, \name still reaches $0.883$ for alignment and $0.695$ for merging, compared with $0.866$ and $0.657$ when the ceiling is removed entirely.

The reason is that the ceiling mainly provides a coarse ordering between paradigms, while the ease distinguishes individual model pairs within each paradigm. Consequently, the exact numerical ceiling values are not critical as long as their relative ordering is approximately correct. When we deliberately destroy that ordering by shuffling the ceiling values across paradigms, the ceiling stops helping altogether: alignment falls to $0.675$, below the $0.866$ that ease alone reaches with no ceiling, and merging falls to $0.654$ against $0.657$ for ease alone. A wrongly ordered ceiling is therefore worse than no ceiling for alignment and worth nothing for merging. \name needs a directionally correct group ordering, not a precise model-by-model estimate.

\FloatBarrier

\subsection{Pair-specific signal within model families}
\label{app:within_family}

A possible alternative explanation for the strong predictive performance of \name is that it merely separates broad model paradigms with already-known downstream behavior. We therefore repeat the alignment analysis within individual vision-model families. For each row in Table~\ref{tab:within_family}, the vision family is held fixed and the correlation is computed only across candidate pairs within that subset. For example, the CLIP row contains only the 12 candidate pairs whose vision encoder belongs to the CLIP family; variation then comes from the text-model partner and, where available, the vision-model scale.

\begin{table}[!htbp]
\centering
\small
\caption{{Predictive power} within individual vision-model families on Flickr30k I2T retrieval.}
\label{tab:within_family}
\begin{tabular}{lcccc}
\toprule
Vision family held fixed &
Pairs &
Ease vs.\ R@1 &
\name vs.\ R@1 &
Within-family ceiling spread \\
\midrule
CLIP only            & 12 & 0.67 & 0.78 & 0.013 \\
SigLIP only          & 12 & 0.85 & 0.87 & 0.023 \\
CLIP fine-tuned only & 10 & 0.83 & 0.83 & 0.017 \\
MAE only             & 12 & 0.87 & 0.87 & 0.016 \\
DINOv2 only          & 13 & 0.77 & 0.80 & 0.024 \\
Classification only  & 13 & 0.43 & 0.56 & 0.009 \\
\bottomrule
\end{tabular}
\end{table}

Ease is positively correlated with downstream performance in all six families, with a median Spearman correlation of $0.80$. The encoder-specific ceiling varies little within each family, so performance differences within a row are driven primarily by pair-specific agreement rather than by broad differences between paradigms.

\FloatBarrier

\subsection{Sensitivity to training stage within a fixed architecture}
\label{app:training_stage}

We further test whether \name tracks changes induced by additional training rather than only static architectural or paradigm labels. Our library contains OpenCLIP LAION-2B checkpoints at three scales together with versions subsequently fine-tuned for ImageNet-12K classification. Table~\ref{tab:training_stage} compares their cross-modal ceilings and mean downstream retrieval before and after classification fine-tuning.

\begin{table}[!htbp]
\centering
\small
\caption{Change in cross-modal ceiling and mean Flickr30k I2T R@1 after ImageNet-12K classification fine-tuning.}
\label{tab:training_stage}
\begin{tabular}{lcccc}
\toprule
Encoder &
Ceiling before ft &
Ceiling after ft &
Mean R@1 before &
Mean R@1 after \\
\midrule
CLIP ViT-B/16 & 0.138 & 0.109 & 0.396 & 0.319 \\
CLIP ViT-L/14 & 0.147 & 0.117 & 0.579 & 0.548 \\
CLIP ViT-H/14 & 0.151 & 0.127 & 0.628 & 0.610 \\
\bottomrule
\end{tabular}
\end{table}

Classification fine-tuning lowers the cross-modal ceiling at all three scales, and downstream retrieval changes in the same direction. The diagnostic therefore tracks a training-stage change within a fixed architecture family, rather than merely recovering a difference between named pretraining paradigms.

\FloatBarrier

\subsection{SigLIP text-tower readout}
\label{app:siglip_text}

The unusually low overlap visible for several points in \Cref{fig:phase} is associated with the SigLIP \emph{text encoder}, rather than with the SigLIP vision-text pairing specifically.  Across all vision encoders paired with this text tower, calibrated cross-modal overlap lies between 0.017 and 0.031, compared with approximately 0.10-0.17 against the other language-model families.
The layer search also behaves differently for this tower. It selects the first SigLIP text layer, corresponding to the embedding-level representation, whereas the other text towers select deeper layers between 11 and 35. Before calibration, the selected SigLIP representation has overlap 0.041, compared with approximately 0.10-0.15 for the other text encoders.
Our language-side representation is the mean-pooled hidden state. Under this readout, the SigLIP text tower therefore exposes substantially less local neighborhood structure to the probe than the other text encoders. This observation should be interpreted as specific to the representation readout used in our experiments rather than as a general statement about SigLIP representations.
The downstream behavior is consistent with this measurement: the seven candidate pairs using the SigLIP text tower have mean Flickr30k R@1 of $0.205$, compared with approximately $0.28-0.42$ for the other text-model families. Importantly, ease still ranks the seven SigLIP-text pairs internally, with Spearman $\rho=0.75$. \name therefore flags these pairs as relatively weak under the chosen text-tower readout while retaining pair-specific predictive signal.

\subsection{What the ceiling predicts under increasing training budget}
\label{app:training_budget}

The ceiling is intended to capture the level of compatibility supported by the application, rather than the amount of improvement obtained from increasing the optimization budget by an arbitrary amount. To examine this interpretation directly, we sweep 10, 30, 100, and 300 alignment epochs on 18 vision-language pairs spanning OLMo-1B, OLMo-7B, and BLOOMz-3B text towers and six vision encoders. The OLMo-1B sweep is reported in Table~\ref{tab:budget_sweep}.

\begin{table}[!htbp]
\centering
\small
\caption{Flickr30k I2T R@1 as the alignment-training budget increases, with OLMo-1B held fixed.}
\label{tab:budget_sweep}
\begin{tabular}{lccccc}
\toprule
Vision encoder & Ceiling & 10 ep & 30 ep & 100 ep & 300 ep \\
\midrule
MAE ViT-B              & 0.055 & 0.142 & 0.183 & 0.216 & 0.215 \\
AugReg ViT-B           & 0.101 & 0.269 & 0.356 & 0.411 & 0.420 \\
DINOv2 ViT-B           & 0.106 & 0.362 & 0.480 & 0.477 & 0.471 \\
CLIP ViT-B ft-IN12k    & 0.109 & 0.288 & 0.370 & 0.454 & 0.459 \\
SigLIP ViT-B           & 0.130 & 0.359 & 0.472 & 0.544 & 0.559 \\
CLIP ViT-B LAION-2B    & 0.138 & 0.424 & 0.528 & 0.587 & 0.607 \\
\bottomrule
\end{tabular}
\end{table}

The ceiling-limited pair stops improving: MAE gains $0.074$ R@1 between 10 and 100 epochs and then nothing at all, $0.216$ to $0.215$, for a further threefold increase in budget. The high-ceiling pairs are still climbing at 300 epochs, CLIP by $0.020$ and SigLIP by $0.015$ over their 100-epoch values. Ranking the pairs by outcome therefore converges on ranking them by ceiling as the budget grows: Spearman between outcome and ceiling is $0.83$ at 10 and 30 epochs and $0.94$ at 100 and 300.

The same holds across text towers. Pooled over all 18 pairs, the correlation between outcome and ceiling rises from $0.79$ at 10 epochs to $0.92$ at 300. At 300 epochs, it is $0.94$, $0.94$, and $0.89$ within the OLMo-1B, OLMo-7B, and bloomz-3b towers, respectively.

The ceiling predicts the performance level a model pair converges to, and not necessarily how much performance increases when the budget does. The gain from 10 to 300 epochs also depends on how close each pair already is to its plateau at 10 epochs. For OLMo-1B, all pairs start sufficiently far from saturation, so pairs with higher ceiling also improve more, yielding a correlation of $0.89$. For bloomz-3b, several high-ceiling pairs are already close to their plateau after 10 epochs and therefore show little additional improvement, while the lower-ceiling MAE pair still has more room to improve. This produces a negative correlation between ceiling and gain. Our claim is that the ceiling predicts the final plateau, not the improvement observed between two specific training budgets.

\FloatBarrier

%% file: checklist.tex
\section*{NeurIPS Paper Checklist}

\begin{enumerate}

\item {\bf Claims}
    \item[] Question: Do the main claims made in the abstract and introduction accurately reflect the paper's contributions and scope?
    \item[] Answer: \answerYes{} 
    \item[] Justification: Yes, the main claims in the abstract and introduction are supported by the theoretical results in \Cref{sec:theory}, and by empirical evaluations in \Cref{sec:practical} across cross-modal alignment and task-arithmetic merging.

\item {\bf Limitations}
    \item[] Question: Does the paper discuss the limitations of the work performed by the authors?
    \item[] Answer: \answerYes{} 
    \item[] Justification: Yes, the paper discusses the limitations in \Cref{app:limitations}.

\item {\bf Theory assumptions and proofs}
    \item[] Question: For each theoretical result, does the paper provide the full set of assumptions and a complete (and correct) proof?
    \item[] Answer: \answerYes{} 
    \item[] Justification: Yes, for all the theoretical results in \Cref{sec:theory}, the paper provides a full set of assumptions and complete proofs in \Cref{app:distribution_identities},  \Cref{app:ease}, and \Cref{app:ceiling}.

    \item {\bf Experimental result reproducibility}
    \item[] Question: Does the paper fully disclose all the information needed to reproduce the main experimental results of the paper to the extent that it affects the main claims and/or conclusions of the paper (regardless of whether the code and data are provided or not)?
    \item[] Answer: \answerYes{} 
    \item[] Justification: Yes, the experimental details, along with all the hyperparameters are described in the experimental setup in \Cref{app:setup}.

\item {\bf Open access to data and code}
    \item[] Question: Does the paper provide open access to the data and code, with sufficient instructions to faithfully reproduce the main experimental results, as described in supplemental material?
    \item[] Answer: \answerYes{} 
    \item[] Justification: Yes, the code is available at \href{https://github.com/mlbio-epfl/COMPASS}{https://github.com/mlbio-epfl/COMPASS}

\item {\bf Experimental setting/details}
    \item[] Question: Does the paper specify all the training and test details (e.g., data splits, hyperparameters, how they were chosen, type of optimizer) necessary to understand the results?
    \item[] Answer: \answerYes{} 
    \item[] Justification: Yes, the experimental details and all the hyperparameters are given in \Cref{app:setup}.

\item {\bf Experiment statistical significance}
    \item[] Question: Does the paper report error bars suitably and correctly defined or other appropriate information about the statistical significance of the experiments?
    \item[] Answer: \answerYes{} 
    \item[] Justification: Yes, the paper reports Pearson correlations with significance levels for the predictor analysis in \Cref{sec:practical}, and Kruskal--Wallis tests for the regime-level comparison in \Cref{sec:practical_merge} and \Cref{app:merge_extras}. 

\item {\bf Experiments compute resources}
    \item[] Question: For each experiment, does the paper provide sufficient information on the computer resources (type of compute workers, memory, time of execution) needed to reproduce the experiments?
    \item[] Answer: \answerYes{} 
    \item[] Justification: Yes, \Cref{app:setup_compute} reports GPU type, memory, per-run wall time, per-run memory use, and the approximate total compute budget for the reported experiments.
    
\item {\bf Code of ethics}
    \item[] Question: Does the research conducted in the paper conform, in every respect, with the NeurIPS Code of Ethics \url{https://neurips.cc/public/EthicsGuidelines}?
    \item[] Answer: \answerYes{} 
    \item[] Justification: Yes, the authors have carefully reviewed and adhered to the NeurIPS Code of Ethics.

\item {\bf Broader impacts}
    \item[] Question: Does the paper discuss both potential positive societal impacts and negative societal impacts of the work performed?
    \item[] Answer: \answerYes{} 
    \item[] Justification: Yes, the positive societal impact of the paper is discussed in \Cref{sec:intro}, \Cref{sec:practical}, and \Cref{sec:discussion} and broader impact is discussed in \Cref{app:broader_impact}. Negative societal impacts are not expected as the outlined diagnostics tends to reduce wasted compute in model transfer experiments.
    
\item {\bf Safeguards}
    \item[] Question: Does the paper describe safeguards that have been put in place for responsible release of data or models that have a high risk for misuse (e.g., pre-trained language models, image generators, or scraped datasets)?
    \item[] Answer: \answerNA{} 
    \item[] Justification: This paper poses no such risks as beyond code we are not releasing any new artifacts such as data or models.

\item {\bf Licenses for existing assets}
    \item[] Question: Are the creators or original owners of assets (e.g., code, data, models), used in the paper, properly credited and are the license and terms of use explicitly mentioned and properly respected?
    \item[] Answer: \answerYes{} 
    \item[] Justification: Yes, the paper cites the pretrained models and datasets used in the experiments, and \Cref{app:setup_licenses} lists the datasets together with their licenses or terms of use.

\item {\bf New assets}
    \item[] Question: Are new assets introduced in the paper well documented and is the documentation provided alongside the assets?
    \item[] Answer: \answerNA{} 
    \item[] Justification: This paper does not release any new assets beyond the code, which is adequately documented in order to reproduce all the results from the paper.

\item {\bf Crowdsourcing and research with human subjects}
    \item[] Question: For crowdsourcing experiments and research with human subjects, does the paper include the full text of instructions given to participants and screenshots, if applicable, as well as details about compensation (if any)? 
    \item[] Answer: \answerNA{} 
    \item[] Justification: The paper does not involve crowdsourcing nor research with human subjects.

\item {\bf Institutional review board (IRB) approvals or equivalent for research with human subjects}
    \item[] Question: Does the paper describe potential risks incurred by study participants, whether such risks were disclosed to the subjects, and whether Institutional Review Board (IRB) approvals (or an equivalent approval/review based on the requirements of your country or institution) were obtained?
    \item[] Answer: \answerNA{} 
    \item[] Justification: The paper does not involve human-subject research, so IRB approval or equivalent review is not applicable.

\item {\bf Declaration of LLM usage}
    \item[] Question: Does the paper describe the usage of LLMs if it is an important, original, or non-standard component of the core methods in this research? Note that if the LLM is used only for writing, editing, or formatting purposes and does \emph{not} impact the core methodology, scientific rigor, or originality of the research, declaration is not required.
    \item[] Answer: \answerNA{} 
    \item[] Justification: The paper does not involve any development or LLMs. The use of LLM was limited to writing, editing, formatting or visualization assistance and did not affect the scientific methodology, experiments, or claims.

\end{enumerate}